\documentclass[accepted]{uai2026} 
\usepackage{xcolor}

\usepackage{natbib}
\usepackage{mathtools} 
\usepackage{amsthm, amsfonts, amssymb}
\usepackage{graphicx}
\usepackage{cleveref}
\crefname{assumption}{Assumption}{Assumptions}
\Crefname{assumption}{Assumption}{Assumptions}
\usepackage{algorithm}
\usepackage{algpseudocode}
\usepackage{booktabs}
\usepackage{bm}
\usepackage{bbm}
\usepackage{multirow}
\usepackage{subcaption}

\theoremstyle{plain}
\newtheorem{theorem}{Theorem}
\newtheorem{proposition}[theorem]{Proposition}
\newtheorem{lemma}[theorem]{Lemma}
\newtheorem{corollary}[theorem]{Corollary}
\theoremstyle{definition}

\newtheorem{assumption}{Assumption}
\theoremstyle{remark}

\newcommand{\R}{\mathbb{R}}
\newcommand{\E}{\mathbb{E}}
\newcommand{\Var}{\mathrm{Var}}
\newcommand{\prob}{\mathbb{P}}
\newcommand{\argmin}{\operatornamewithlimits{arg\,min}}

\newcommand{\ind}{\perp\!\!\!\perp}
\newcommand{\bU}{\bm{U}}
\newcommand{\bV}{\bm{V}}
\newcommand{\bZ}{\bm{Z}}
\newcommand{\bX}{\bm{X}}
\newcommand{\bW}{\bm{W}}
\newcommand{\bh}{\bm{h}}
\newcommand{\cF}{\mathcal{F}}
\newcommand{\cG}{\mathcal{G}}
\newcommand{\cH}{\mathcal{H}}
\newcommand{\cD}{\mathcal{D}}
\newcommand{\cM}{\mathcal{M}}
\newcommand{\cR}{\mathcal{R}}
\newcommand{\cO}{\mathcal{O}}
\newcommand{\cL}{\mathcal{L}}

\newcommand{\cP}{\mathcal{P}}
\newcommand{\Rad}{\mathfrak{R}}

\newcommand{\Dr}[1]{\Delta^2_{2,r}\!\left(#1\right)}

\newcommand{\CALM}{\textsc{Calm}}
\newcommand{\SROSCAR}{\textsc{SR-Oscar}}
\newcommand{\MROSCAR}{\textsc{MR-Oscar}}
\newcommand{\ROSCAR}{\textsc{R-Oscar}}
\newcommand{\RACER}{\textsc{Racer}}

\title{Improving RCT-Based Treatment Effect Estimation Under Covariate Mismatch via Calibrated Alignment\thanks{Accepted at the 42nd Conference on Uncertainty in Artificial Intelligence (UAI 2026).}}

\author[1]{\href{mailto:amir.asiaeetaheri@vumc.org?Subject=Your UAI 2026 paper}{Amir Asiaee}}
\author[1]{Samhita Pal}
\affil[1]{%
    Department of Biostatistics,
    Vanderbilt University Medical Center,
    TN 37203, USA
}

\begin{document}
\maketitle

\begin{abstract} 
Randomized controlled trials (RCTs) are the gold standard for estimating treatment effects, yet they are often underpowered for detecting effect heterogeneity.
Large observational studies (OS) can supplement RCTs for conditional average treatment effect (CATE) estimation, but a key barrier is \emph{covariate mismatch}: the two sources measure different, only partially overlapping, covariates.
We propose \CALM{} (\textbf{C}alibrated \textbf{AL}ignment under covariate \textbf{M}ismatch), which learns embeddings that map each source's features into a common representation space.
OS outcome models are transferred to the RCT embedding space and calibrated using trial data, preserving causal identification from randomization.
Finite-sample risk bounds decompose into \emph{alignment error}, \emph{outcome-model complexity}, and \emph{calibration complexity} terms, making explicit when the learned embedding is accurate enough to reduce variance.
We instantiate \CALM{} in two forms: a closed-form linear version, \CALM-Lin{}, and a neural representation-learning version, \CALM-NN{}. 
Across 51 simulation settings, calibration-based linear methods are effectively tied in linear-CATE regimes, while CALM-NN wins all 22 nonlinear-CATE settings by wide margins.
Moreover, on two real-data studies \CALM-NN{} delivers the largest gains over the trial-only baseline.
\end{abstract}

\section{Introduction}\label{sec:intro}

Heterogeneous treatment effects (HTEs) are central to precision medicine: understanding how individual patient characteristics modulate treatment response {helps clinicians match interventions to patients} \citep{kosorok2019precision}.
RCTs remain the gold standard for causal inference, providing unconfounded treatment comparisons that support consistent estimation of the conditional average treatment effect,
\[
\tau^r(\bm{x}) = \E\bigl[Y(1) - Y(-1) \mid \bX = \bm{x},\, S = r\bigr],
\]
where $Y(a)$ denotes the potential outcome under treatment $a\in\{-1,1\}$ and $S=r$ indicates the RCT population.
In practice, however, most RCTs are powered for average effects rather than the fine-grained heterogeneity that guides individualized decisions \citep{wang2007statistics, kent2020path}.
At the same time, large-scale observational studies---electronic health records (EHRs), registries, and insurance claims---provide rich covariates and {large sample sizes}, but are vulnerable to confounding.

{This contrast motivates borrowing from OS data to improve} the precision of within-trial CATE estimation.
\citet{asiaee2025improving} formalize this $\cO\to\cR$ direction by introducing the \emph{counterfactual mean outcome} (CMO) as the variance-minimizing augmentation function for pseudo-outcome regression, and propose \ROSCAR, a two-stage calibration procedure that learns outcome models from the OS, calibrates them to the RCT, and uses the calibrated predictions as an augmentation function.
Thus, OS data are used to improve the CMO augmentation rather than to identify the treatment effect: better CMO estimates make the pseudo-outcome regression more efficient, while the CATE target remains anchored in the randomized trial.

\textbf{The covariate mismatch barrier.}\ 
A major obstacle to RCT--OS integration is that the covariates collected in trials and observational sources rarely align---a problem we call \emph{covariate mismatch}.
The RCT may contain behavioral or survey-based variables absent from the OS, while the OS may contain laboratory or utilization features not recorded in the trial \citep{Rassler2002}.
This differs fundamentally from \emph{covariate shift} \citep{sugiyama2012machine}, where the same variables are observed but follow different distributions.
Under mismatch, the outcome models trained on OS features cannot be directly evaluated on RCT units, {so the transfer step in data-borrowing methods is no longer well defined}.
Indeed, for generalization ($\cR\to\cO$), \citet{colnet2021generalizing} show that oracle linear imputation of a key covariate observed in only one source does not remove the missing-covariate bias, even in a Gaussian linear-CATE setting.

{\citet{pal2026mismatch} adapt \ROSCAR{} to covariate mismatch by partitioning covariates into shared ($\bZ \in \R^{p_z}$), RCT-only ($\bU \in \R^{p_u}$), and OS-only ($\bV \in \R^{p_v}$) blocks. Their imputation-based estimator, \MROSCAR{}, learns a map from $\bZ$ to $\bV$ in the OS and fills in the structurally missing $\bV$ block in the RCT before applying \ROSCAR{}. They also define \SROSCAR{} (shared R-OSCAR), which applies \ROSCAR{} using only the shared covariates $\bZ$.}

\textbf{Imputation is a harder problem than needed.}\ 
The imputation approach requires reconstructing the entire $\bV$ vector in the RCT---a potentially high-dimensional regression problem whose difficulty scales with $p_v$ and the complexity of $P(\bV \mid \bZ)$.
{For CATE estimation, however, the target is only a representation} that is sufficient for (i)~predicting outcomes and (ii)~estimating discrepancy functions connecting outcome means across the two populations.
When the outcome-relevant information in $\bV$ lies on a low-dimensional manifold, the imputation approach pays an unnecessarily high price by attempting to recover the full $\bV$ rather than its outcome-relevant projection.

\textbf{Our contribution: embedding alignment.}\ 
Our proposed \textbf{C}alibrated \textbf{AL}ignment under covariate \textbf{M}ismatch (\CALM{}) replaces imputation with \emph{representation alignment}.
Instead of reconstructing missing covariates, we learn embedding functions $\phi^o:\R^{p_o}\to\R^d$ and $\phi^r:\R^{p_r}\to\R^d$ that map the heterogeneous feature spaces into a common $d$-dimensional representation space $\cH$.
Outcome models are trained in the OS embedding space and then transferred to the RCT embedding space, where they are calibrated and used for CATE estimation {through} the double-calibration pipeline of \ROSCAR{}.
OS data only inform the pseudo-outcome used to reduce variance of the CATE estimate (see \Cref{fig:calm-pipeline} for an overview); causal identification rests entirely on randomization within the trial and remains intact.

Our contributions are:
\begin{enumerate}[leftmargin=16pt, itemsep=0pt]
\item \textbf{Embedding-alignment framework.} We introduce \CALM, which integrates representation learning into the \ROSCAR{} calibration pipeline, replacing imputation with learned embeddings. The framework inherits {the negative-transfer protection from double calibration}.

\item \textbf{Finite-sample risk bounds.} We derive non-asymptotic bounds that decompose into alignment error, outcome-model, calibration, and CATE-class complexity terms, {identifying conditions under which embedding outperforms imputation}. {In linear Gaussian models, \CALM-Lin recovers oracle linear imputation as a special case.}

\item \textbf{Extensive experiments.} {We evaluate linear and neural-network instantiations} across 51 simulation settings, a semi-synthetic benchmark using real trial covariates, and two real-data studies (the Greenlight Plus trial linked to an EHR cohort, and a semi-real Tennessee STAR analysis).
\end{enumerate}

\textbf{Relation to \ROSCAR{} and \MROSCAR{}.}\ 
\ROSCAR{}~\citep{asiaee2025improving} borrows from observational data when both sources are measured in the same covariate space. With covariate mismatch, \MROSCAR{}~\citep{pal2026mismatch} first imputes the covariates missing from the trial and then applies this borrowing pipeline. \CALM{} instead applies the same idea in a learned common embedding, avoiding full reconstruction of the missing covariates. Our theory and experiments show when this embedding route is preferable to imputation or to using only the shared covariates: when the embedding keeps outcome-relevant signal while reducing dimension and aligning the two sources well.

\section{Related Work}\label{sec:related}

CATE estimation {has been studied through} meta-learners \citep{kunzel2019metalearners, kennedy2023towards, nie2021quasi}, causal forests \citep{athey_generalized_2019}, Bayesian methods \citep{hahn2020bayesian}, and representation-learning approaches \citep{johansson2016learning, shalit2017estimating, yao2018representation, shi2019adapting, hassanpour2020learning}.
{Most RCT--OS integration work targets the $\cR\to\cO$ (generalizability) direction \citep{colnet2024causal, degtiar2023review, dahabreh2020extending}; recent work in that direction also models outcome-mean shifts directly rather than only covariate shifts \citep{asiaee2026sharp, asiaee2026omitted}. In the opposite $\cO\to\cR$ direction, \citet{cheng2021adaptive} adaptively weight CATE estimates from both sources, \citet{oberst2022understanding} study optimal biased--unbiased combinations, and \citet{karlsson2025robust} provide worst-case guarantees. The \ROSCAR{} framework of \citet{asiaee2025improving} introduces double calibration for borrowing from OS data in a shared covariate space, while \citet{pal2026mismatch} extend this route to covariate mismatch through imputation (\MROSCAR{}) and shared-covariate borrowing (\SROSCAR{}). We use these as baselines and introduce \CALM{}, which borrows in a learned embedding rather than by imputing missing covariates.}
In heterogeneous transfer learning, standard domain-adaptation methods \citep{ben2010theory, ganin2016domain, long2015learning, pan2010survey, weiss2016survey} assume shared features. {Cross-domain methods instead use} CCA variants \citep{hardoon2004canonical, andrew2013deep} or shared-private encoders \citep{bousmalis2016domain}. In the causal setting, \citet{bica2022transfer} propose HTCE-learners, which require unconfoundedness in both domains, lack calibration-based negative-transfer protection, and provide no finite-sample guarantees.
Sufficient dimension reduction \citep{cook2007sufficient, li2018sufficient, ma2012semiparametric, ghosh2021sufficient} also targets outcome-relevant projections but {does not handle} cross-dataset covariate mismatch.
Bayesian borrowing via power priors \citep{ibrahim2000power} and commensurate priors \citep{hobbs2011hierarchical} discounts external data through a power parameter; a Bayesian, bias-limited version of this borrowing for the same covariate-mismatch setting is developed in \citep{asiaee2026bcalmbiaslimitedbayesianborrowing}, and {we instead use} a frequentist setting with calibration-based correction.
Negative transfer is a well-documented risk \citep{rosenstein2005transfer}; recent causal-inference work addresses it through optimal estimator combinations \citep{oberst2022understanding}, worst-case guarantees \citep{karlsson2025robust}, and calibration \citep{asiaee2025improving}. \CALM{} inherits {this calibration-based protection and adds} representation alignment.

\vspace{-.25em}
\section{Problem Setup and Background}\label{sec:setup}

\subsection{Notation and Data Structure}

{There are} two data sources: a randomized controlled trial ($S = r$) and a large observational study ($S = o$).
Let $Y$ denote the outcome of interest, $A \in \{-1, 1\}$ the binary treatment indicator, and $Y(a)$ the potential outcome under treatment $a$.
We assume consistency: $Y = Y(A)$.

{The two sources measure different subsets} of baseline covariates: the RCT records $\bX^r$ and the OS records $\bX^o$.
We denote the covariates measured in both sources by $\bZ\in\R^{p_z}$ (shared), those exclusive to the RCT by $\bU:=\bX^r\!\setminus\!\bZ\in\R^{p_u}$, and those exclusive to the OS by $\bV:=\bX^o\!\setminus\!\bZ\in\R^{p_v}$, so that $\bX^r=(\bU,\bZ)\in\R^{p_r}$ and $\bX^o=(\bZ,\bV)\in\R^{p_o}$, where $p_r=p_u+p_z$ and $p_o=p_z+p_v$.
We write $\bX=(\bU,\bZ,\bV)\in\R^{p}$ with $p=p_u+p_z+p_v$ for the complete covariate vector, and use lowercase $\bm{x}^r=(\bm{u},\bm{z})$, $\bm{x}^o=(\bm{z},\bm{v})$ for generic realizations.

The observed data are $\bigl\{(\bX^r_i, A_i, Y_i)\bigr\}_{i=1}^{n^r}$ from the RCT and $\bigl\{(\bX^o_j, A_j, Y_j)\bigr\}_{j=1}^{n^o}$ from the OS, with $n^r \ll n^o$ in the typical regime of interest.
Let $n_a^s$ denote the number of units in arm $a$ under source $s$. The RCT propensity is $\pi^r_a(\bm{x}^r) = \prob(A = a \mid \bX^r = \bm{x}^r, S = r)$.
For each source $s\in\{o,r\}$, write
\[
\mu^s_a(\bm{x}^s)=\E[Y\mid \bX^s=\bm{x}^s,A=a,S=s].
\]
In the RCT, randomization implies $\mu^r_a(\bm{x}^r)=\E[Y(a)\mid \bX^r=\bm{x}^r,S=r]$.
Our inferential target is the CATE in the RCT population, marginalized over the unobserved $\bV$:
\begin{equation}\label{eq:target-cate} \small
\tau^r(\bm{x}^r) = \mu^r_1(\bm{x}^r)-\mu^r_{-1}(\bm{x}^r).
\end{equation}

\subsection{The CMO-Augmented Pseudo-Outcome Framework}

Following \citet{asiaee2025improving}, we construct pseudo-outcomes using an augmentation function $m:\R^{p_r} \to \R$:
\begin{equation}\label{eq:pseudo-outcome}
\tau_m(\bX^r, A, Y) = \frac{A\bigl(Y - m(\bX^r)\bigr)}{\pi^r_A(\bX^r)}.
\end{equation}
Under the standard RCT identification assumptions (SUTVA, ignorability, positivity), $\E[\tau_m \mid \bX^r, S = r] = \tau^r(\bX^r)$ for \emph{any} $m$, so the pseudo-outcome {has the CATE as its conditional mean}.
The CATE is then estimated by minimizing the empirical squared loss over a class of candidate functions $\cF$:
\begin{multline}\label{eq:cate-erm} \small
\hat{\tau}^r \in \argmin_{f \in \cF} \frac{1}{n^r} \sum_{i: S_i = r}
\bigl(\tau_m(\bX^r_i, A_i, Y_i) - f(\bX^r_i)\bigr)^2.
\end{multline}

{A key result} of \citet{asiaee2025improving} is that {the augmentation $m$ controls} the CATE estimation risk.
{They show that the prediction risk decomposes as} \citep{asiaee2025improving}:
{\small\begin{equation}\label{eq:risk-decompose} 
R_m(\hat{\tau}) = \underbrace{\E_{\bX^r}\bigl[\Var(\tau_m \mid \bX^r)\bigr]}_{\text{irreducible error}}
+ \underbrace{\Delta_2^2(\hat{\tau}, \tau^r)}_{\text{CATE estimation error}},
\end{equation}}
where $\Delta_2^2(f, g) = \E[(f(\bX^r) - g(\bX^r))^2 \mid S = r]$ is the mean squared error (MSE).
They show that both terms are controlled by $\Delta_2^2(m, \tilde{\mu}^r)$, where $\tilde{\mu}^r$ is the \emph{counterfactual mean outcome}:
\begin{equation}\label{eq:cmo-def}
\tilde{\mu}^r(\bm{x}^r) = \textstyle\sum_{a} \pi^r_{-a}(\bm{x}^r) \, \mu^r_a(\bm{x}^r).
\end{equation}
In other words, the best strategy is to accurately estimate the CMO $\tilde{\mu}^r$, use it as the augmentation function $m$ in \eqref{eq:pseudo-outcome}, and then optimize the objective in \eqref{eq:cate-erm}. {This motivates OS borrowing specifically for improved CMO estimation}.

\subsection{The \ROSCAR{} Pipeline and Imputation-Based Baselines}\label{sec:roscar-review}

{Here, we review the \ROSCAR{} pipeline \citep{asiaee2025improving} and the mismatch baselines of \citet{pal2026mismatch}.}

\textbf{\ROSCAR{} without covariate mismatch.}\ 
When there is no covariate mismatch (i.e., $p_u = p_v = 0$), the \ROSCAR{} pipeline {has three steps}:

\begin{enumerate}[leftmargin=16pt, itemsep=2pt]
\item \textbf{OS outcome model:} For each arm $a$, learn the regression function $\hat{\mu}^o_a(\bm{x})$ from OS data.
\item \textbf{Outcome calibration:} Estimate a discrepancy function $\hat{\delta}^t_a(\bm{x})$ from RCT data so that $\hat{\mu}^o_a(\bm{x}) + \hat{\delta}^t_a(\bm{x}) \approx \mu^r_a(\bm{x})$.
The calibrated CMO is $\hat{m}(\bm{x}) = \sum_a \pi^r_{-a}(\bm{x})[\hat{\mu}^o_a(\bm{x}) + \hat{\delta}^t_a(\bm{x})]$.
\item \textbf{CATE calibration:} Using $\hat{m}$ to form pseudo-outcomes in \eqref{eq:pseudo-outcome} and estimate a CATE correction $\hat{\delta}(\bm{x})$ in \eqref{eq:cate-erm}, yielding $\hat{\tau}_{\ROSCAR}(\bm{x}) = \sum_a a[\hat{\mu}^o_a(\bm{x}) + \hat{\delta}^t_a(\bm{x})] + \hat{\delta}(\bm{x})$.
\end{enumerate}

The two-stage design decouples nuisance estimation from CATE estimation: misspecification in outcome models affects only the augmentation quality (variance), while CATE consistency depends only on correct specification of the CATE discrepancy $\hat{\delta}$.

{Under covariate mismatch, we compare against the two baselines of \citet{pal2026mismatch}. \MROSCAR{} first imputes the OS-only block by estimating $\hat{g}:\R^{p_z}\to\R^{p_v}$ in the OS, sets $\widehat{\bV}_i=\hat{g}(\bZ_i)$ for RCT units, and then runs \ROSCAR{} on $(\bX^r,\widehat{\bV})$. \SROSCAR{} (shared R-OSCAR) skips imputation and runs \ROSCAR{} only on the shared covariates $\bZ$. The relevant feature of the \MROSCAR{} bound is that it pays both an imputation-error term in the missing block and outcome-model complexity in the full imputed covariate space; the comparison in Section~\ref{sec:theory} shows how \CALM{} replaces these costs with alignment and sufficiency terms in a lower-dimensional embedding.}

\section{\CALM: Embedding-Aligned CATE Estimation}\label{sec:method}

This section presents \CALM{}, our method for addressing covariate mismatch through representation alignment rather than imputation.

\begin{figure}[t]
  \centering
  \includegraphics[width=\linewidth]{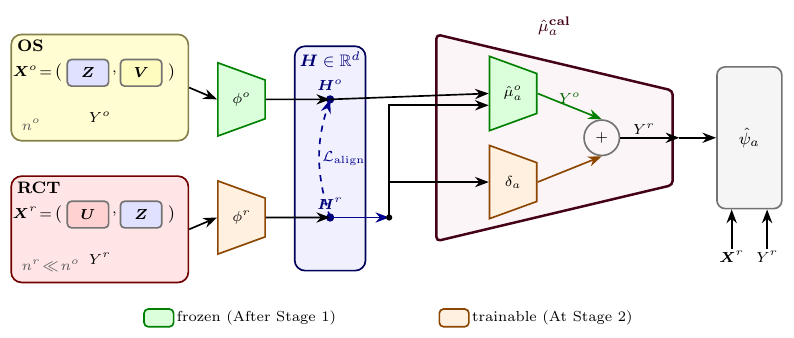}
  \caption{How OS data inform pseudo-outcome construction in \CALM{}.
  OS and RCT data pass through source-specific encoders $\phi^o$ (frozen after Stage~1) and $\phi^r$ (trainable at Stage~2) into a shared embedding space $\bm{H}\in\R^d$.
  The calibrated outcome model $\hat{\mu}^r_a = \hat{\mu}^o_a + \delta^t_a$ combines the frozen OS outcome head with a learnable shift.
  Pseudo-outcomes $\hat{\psi}^r_a$ are then constructed from these calibrated predictions together with the RCT covariates and outcomes.
  The calibrated model $\hat{\mu}^r_a$ is also used in the subsequent CATE calibration stage (Stage~4 of Algorithm~\ref{alg:calm}).}
  \label{fig:calm-pipeline}
\end{figure}

\subsection{Core Idea: From Imputation to Alignment}

The \ROSCAR{} pipeline requires not the raw covariates but rather \emph{a common representation} in which (i)~outcome models can be meaningfully evaluated for both OS and RCT units, and (ii)~discrepancy functions can be estimated.
Imputation {constructs such a representation} by reconstructing the missing $\bV$ block, but {this intermediate step solves} a harder problem than necessary.

\CALM{} instead learns embedding functions $\phi^o : \R^{p_o} \to \R^d$ (OS encoder) and $\phi^r : \R^{p_r} \to \R^d$ (RCT encoder) that map each source's full feature vector into a shared $d$-dimensional space $\cH$.
The OS outcome model is trained on the OS embeddings $\hat{\mu}^o_a(\phi^o(\bX^o))$ and {is evaluated} on the RCT embeddings $\hat{\mu}^o_a(\phi^r(\bX^r))$ during calibration.
What makes this work is that the embeddings are outcome-supervised, not merely geometric. Stage~1 trains the OS encoder $\phi^o$ to predict the outcome from the \emph{full} $\bX^o=(\bZ,\bV)$, so $\phi^o$ is forced to keep the outcome-relevant variation in $\bV$ rather than just compress $\bZ$. Stage~2 then aligns the RCT encoder $\phi^r$ to that outcome-aware target under a $\bZ$-conditioned penalty. The operative notion of closeness is ``produces similar outcome predictions on units with similar shared covariates,'' not ``is geometrically close on the raw shared covariates'': the latter would be satisfied by a degenerate $\bZ$-only encoder, whereas the former would not, since the OS embedding carries $\bV$-relevant signal that $\phi^r$ must reproduce. Under such an alignment, the OS outcome model transfers usefully to the RCT embeddings, and the calibration step corrects the residual discrepancy.

\subsection{Assumptions}\label{sec:assumptions}

{We assume} the standard internal validity of the RCT:

\begin{assumption}[Internal validity of the RCT]\label{as:rct}
SUTVA holds;
$(Y(1), Y(-1)) \ind A \mid (\bX^r, S = r)$; and
there exists $\rho > 0$ such that $\rho \leq \pi^r_a(\bX^r) \leq 1 - \rho$ almost surely for each $a$.
\end{assumption}

{Unlike imputation-based borrowing, we need an embedding-level assumption:}

\begin{assumption}[Outcome-sufficient embedding]\label{as:embed-sufficient}
{
For fixed encoders $\phi^s$, the per-arm outcome means are well-approximated by functions of the embeddings: for each source $s\in\{o,r\}$ and arm $a$, there exists $\tilde{\mu}^s_a:\R^d\to\R$ in a function class $\cM^s_a$ such that
\[
\E\bigl[\bigl(\mu^s_a(\bX^s) - \tilde{\mu}^s_a(\phi^s(\bX^s))\bigr)^2 \mid S = s\bigr] \leq \epsilon_{\mathrm{suff}}^2,
\]
where $\epsilon_{\mathrm{suff}}^2$ is the \emph{sufficiency gap}.
}
\end{assumption}

\vspace{-0.5em}
This is weaker than requiring the embedding to be a sufficient statistic for the outcome: {it permits} an approximation error $\epsilon_{\mathrm{suff}}^2$ that vanishes as $d$ grows.
This outcome-sufficient view has a clear lineage. Embeddings that keep only outcome-relevant directions go back to classical sufficient dimension reduction \citep{li1991sir, cook2007sufficient} and its kernel extensions \citep{fukumizu2009kernel}, and they reappear in causal representation learning \citep{johansson2016learning, shalit2017estimating}. \Cref{as:embed-sufficient} departs from this line in two ways. It is stated per source---a separate condition on the OS and on the trial marginal---rather than as a single requirement on the joint distribution. And it asks the embedding to preserve the conditional outcome mean directly, rather than to balance the treatment groups as in the representation-learning approaches to CATE estimation.

\begin{assumption}[Outcome shift in the embedding space]\label{as:shift-embed}
For each arm $a$, there exists a discrepancy function $\delta_a : \R^d \to \R$ in a function class $\cD_a$ of controlled complexity such that
\[
\tilde{\mu}^r_a(\bh) = \tilde{\mu}^o_a(\bh) + \delta_a(\bh),
\]
where $\tilde{\mu}^s_a$ are the population-level outcome functions operating on the embedding $\bh \in \R^d$.
\end{assumption}

{This is the embedding-space analogue} of the outcome-shift assumption in the \ROSCAR{} framework \citep{asiaee2025improving}.
Outcome means and CATEs {may still differ arbitrarily in the embedding space} between the RCT and OS; the discrepancy is modeled and estimated from RCT data.

\vspace{-0.5em}
\begin{assumption}[Alignment quality]\label{as:alignment}
The alignment error between the two encoders, measured on the RCT distribution, is bounded:
\[
r_\phi^2 := \E\Bigl[\bigl\|\phi^o(\bX^{o,*}) - \phi^r(\bX^r)\bigr\|^2 \;\Big|\; S = r\Bigr] < \infty,
\]
where $\bX^{o,*} = (\bZ, \bV^*)$ with $\bV^*$ denoting the (unobserved) value that $\bV$ would take for an RCT unit.
Furthermore, the outcome model $\tilde{\mu}^o_a$ is $L_\mu$-Lipschitz in its argument and the discrepancy $\delta_a$ is $L_\delta$-Lipschitz.
\end{assumption}

\textbf{Checking the assumptions in practice.}\ 
The bound contains two population quantities that are not directly observable: $r_\phi^2$, because $\bV$ is unobserved in the RCT, and $\epsilon_{\mathrm{suff}}^2$, because it refers to the best outcome-preserving representation in the population. We therefore treat them as diagnostic targets rather than tuning parameters. Appendix~\ref{app:observable-diagnostics} gives held-out proxies: a shared-covariate alignment proxy for $r_\phi^2$ and an OS residual proxy for $\epsilon_{\mathrm{suff}}^2$. These proxies do not prove the assumptions, but they make the theory empirically checkable; in simulations the alignment proxy tracks the oracle error closely, and in the real-data studies we use the same diagnostics to qualify the observed gains. The Lipschitz constants in Assumption~\ref{as:alignment} translate embedding error into prediction error. Appendix~\ref{app:observable-diagnostics} also reports spectral-product upper bounds for the outcome and calibration heads (Table~\ref{tab:cr-lipschitz}); these are regularization-controlled diagnostics, not certified global constants for the full neural network. Exact certification for ReLU networks generally requires specialized norm-constrained architectures or verification methods \citep{anil2019sorting, fazlyab2019efficient}. Training then reduces observable surrogates of $r_\phi^2$ through the alignment objective in Section~\ref{sec:alignment-objectives}.

\subsection{The \CALM{} Algorithm}\label{sec:algorithm}

The \CALM{} procedure {has} four stages.

\textbf{Stage 1: OS outcome model in embedding space.}\ 
Train the OS encoder $\phi^o$ and per-arm outcome heads $\hat{\mu}^o_a : \R^d \to \R$ jointly on OS data:
{\small \begin{multline}\label{eq:stage1}
(\hat{\phi}^o, \hat{\mu}^o_{-1}, \hat{\mu}^o_{+1}) \in \argmin_{\phi^o, \mu_a} \sum_{a \in \{-1,1\}} \frac{1}{n_a^o} \\
\times \sum_{i: A_i = a, S_i = o} \bigl(Y_i - \mu_a(\phi^o(\bX^o_i))\bigr)^2 + \cP_{\mathrm{OS}},
\end{multline}}
where $\cP_{\mathrm{OS}}$ is a penalty (e.g., $\ell_1$, weight decay, or dropout).

\textbf{Stage 2: RCT encoder alignment and outcome calibration.}\ 
Fix $\hat{\phi}^o$ and $\hat{\mu}^o_a$ from Stage~1.
Learn the RCT encoder $\phi^r$ and arm-specific discrepancy functions $\hat{\delta}_a^t : \R^d \to \R$ by minimizing a calibration loss with an alignment penalty:
{\small \begin{multline}\label{eq:stage2}
(\hat{\phi}^r, \hat{\delta}^t_a) \in \argmin_{\phi^r, d_a} \sum_{a} \frac{1}{n_a^r} \sum_{i: A_i = a, S_i = r} \Bigl(Y_i \\
- \hat{\mu}^o_a\bigl(\phi^r(\bX^r_i)\bigr) - d_a\bigl(\phi^r(\bX^r_i)\bigr)\Bigr)^2 \\
+ \cP_{\mathrm{cal}}(d_a) + \lambda \cL_{\mathrm{align}}(\hat{\phi}^o, \phi^r).
\end{multline}}
Here, $\cP_{\mathrm{cal}}$ controls the complexity of the discrepancy (encouraging borrowing from the OS by shrinking $d_a$ toward zero), and $\cL_{\mathrm{align}}$ is {the} alignment loss (see Section~\ref{sec:alignment-objectives}).

\textbf{Stage 3: Calibrated CMO and pseudo-outcome construction.}\ 
Compute the calibrated arm-specific predictions and the calibrated CMO for RCT units:
\begin{align}
\hat{\mu}^{\mathrm{cal}}_a(\bm{x}^r) &= \hat{\mu}^o_a\bigl(\hat{\phi}^r(\bm{x}^r)\bigr) + \hat{\delta}^t_a\bigl(\hat{\phi}^r(\bm{x}^r)\bigr), \label{eq:cal-arm}\\
\hat{m}(\bm{x}^r) &= \sum_{a} \pi^r_{-a}(\bm{x}^r) \, \hat{\mu}^{\mathrm{cal}}_a(\bm{x}^r). \label{eq:cal-cmo}
\end{align}
Form \textbf{pseudo-outcomes} $\hat{\psi}^r_i = A_i(Y_i - \hat{m}(\bX^r_i)) / \pi^r_{A_i}(\bX^r_i)$ for each RCT unit.

\textbf{Stage 4: CATE correction.}\ 
The preliminary CATE is $\tilde{\tau}(\bm{x}^r) = \sum_a a \, \hat{\mu}^{\mathrm{cal}}_a(\bm{x}^r)$.
Estimate a CATE correction $\hat{\delta} : \R^{p_r} \to \R$ using RCT data:
{\small\begin{equation}\label{eq:stage4}
\hat{\delta} \in \argmin_{d \in \cD} \frac{1}{n^r} \sum_{i: S_i = r} \bigl(\hat{\psi}^r_i 
- \tilde{\tau}(\bX^r_i) - d(\bX^r_i)\bigr)^2 + \cP_\tau(d).
\end{equation}}
The \CALM{} estimator is: $\boxed{\hat{\tau}_{\CALM}(\bm{x}^r) = \tilde{\tau}(\bm{x}^r) + \hat{\delta}(\bm{x}^r).}$

The CATE calibration (Stage~4) operates on the \emph{original} RCT covariates $\bX^r$, not on the embedding.
{This keeps} the final CATE estimate interpretable in terms of the original covariates; causal identification is preserved by the unbiasedness of the pseudo-outcomes under RCT randomization.
The full procedure is summarized in Algorithm~\ref{alg:calm}.

\begin{algorithm}[t]
\small 
\caption{\CALM: Calibrated Alignment under Covariate Mismatch}\label{alg:calm}
\begin{algorithmic}[1]
\Require OS data $\{(\bX^o_j, A_j, Y_j)\}_{j=1}^{n^o}$; RCT data $\{(\bX^r_i, A_i, Y_i)\}_{i=1}^{n^r}$; embedding dimension $d$; alignment weight $\lambda$.
\State \textbf{Stage 1:} Train OS encoder $\hat{\phi}^o$ and outcome heads $\hat{\mu}^o_a$ via \eqref{eq:stage1}.
\State \textbf{Stage 2:} Learn RCT encoder $\hat{\phi}^r$ and discrepancies $\hat{\delta}^t_a$ via \eqref{eq:stage2}.
\State \textbf{Stage 3:} Compute calibrated CMO $\hat{m}$ via \eqref{eq:cal-cmo}; form pseudo-outcomes $\hat{\psi}^r_i$.
\State \textbf{Stage 4:} Estimate CATE correction $\hat{\delta}$ via \eqref{eq:stage4}.
\State \Return $\hat{\tau}_{\CALM}(\bm{x}^r) = \tilde{\tau}(\bm{x}^r) + \hat{\delta}(\bm{x}^r)$
\end{algorithmic}
\end{algorithm}

\subsection{Alignment Objectives}\label{sec:alignment-objectives}

The alignment loss $\cL_{\mathrm{align}}$ encourages the two encoders to produce compatible representations.
Since $\bV$ is unobserved in the RCT and $\bU$ in the OS, we cannot directly minimize $\|\phi^o(\bX^o) - \phi^r(\bX^r)\|^2$ on paired units; {we therefore align through} the shared covariates $\bZ$.
{We present two generic alignment objectives; the experiments use the conditional-mean variant described in Appendix~\ref{app:alignment-objectives}.}
\emph{Distribution-matching alignment} (MMD) uses the kernel maximum mean discrepancy \citep{gretton2012kernel}:
\begin{multline}\label{eq:mmd-align}
\cL_{\mathrm{MMD}} = \Bigl\|\tfrac{1}{n^o}\textstyle\sum_{j \in \text{OS}} k\bigl(\hat{\phi}^o(\bX^o_j), \cdot\bigr) \\
- \tfrac{1}{n^r}\textstyle\sum_{i \in \text{RCT}} k\bigl(\phi^r(\bX^r_i), \cdot\bigr)\Bigr\|_{\cH_k}^2,
\end{multline}
which {matches} the marginal embedding distributions (see Appendix~\ref{app:alignment-objectives} for the expanded objective and additional alignment variants).
\emph{$\bZ$-conditioned contrastive alignment} matches units with similar shared covariates:
\begin{multline}\label{eq:contrastive-align}
\cL_{\mathrm{contr}} = \tfrac{1}{n^r}\textstyle\sum_{i \in \text{RCT}} \tfrac{1}{|N_i|} \sum_{j \in N_i} 
\bigl\|\hat{\phi}^o(\bX^o_j) - \phi^r(\bX^r_i)\bigr\|^2,
\end{multline}
where $N_i = \{j \in \text{OS} : \|\bZ_j - \bZ_i\| \leq \varepsilon\}$.
Adversarial alignment \citep{ganin2016domain} is also applicable (Appendix~\ref{app:alignment-objectives}).

\section{Theoretical Analysis}\label{sec:theory}

This section derives finite-sample risk bounds for \CALM{} and compares them to the \MROSCAR{} bounds. {Proofs for the theorem, corollaries, and propositions in this section are in} Appendix~\ref{app:proofs}.

\subsection{Risk Bound for \CALM}

{We use two standard learning-theory quantities in the bound. For a function class $\mathcal{A}$, define its approximation error relative to a target $g$ by
\(
\Delta_2^2(\mathcal{A},g)=\inf_{f\in\mathcal{A}}\Delta_2^2(f,g),
\)
where $\Delta_2^2(f,g)$ is the RCT-population MSE defined in Section~\ref{sec:setup}. We write $\Rad_n(\mathcal{A})$ for the empirical Rademacher complexity of $\mathcal{A}$,
\(
\Rad_n(\mathcal{A})
=
\E_\sigma\!\left[\sup_{f\in\mathcal{A}}\frac{1}{n}\sum_{i=1}^n \sigma_i f(W_i)\right],
\)
where $\sigma_i$ are independent Rademacher signs and $W_i$ denotes the relevant inputs for the class under consideration (RCT covariates, OS embeddings, or RCT embeddings). This quantity measures the statistical complexity of a function class.}


\begin{theorem}[Risk bound for \CALM]\label{thm:calm-main}
Suppose Assumptions~\ref{as:rct}--\ref{as:alignment} hold, {the outcome and fitted-function classes have bounded envelopes,} and all nuisance estimators are obtained via penalized empirical risk minimization with sample splitting or cross-fitting across stages.
Let $\cD$ be the function class for the CATE correction in Stage~4 (Algorithm~\ref{alg:calm}), and let $\cF = \{\tilde{\tau} + d : d \in \cD\}$ be the induced class for the final CATE estimator on $\bX^r$.
Let $\cM^o_a$ be the class for OS outcome models on $\R^d$ and $\cD_a$ the class for arm-specific discrepancy on $\R^d$.
Then, {for any failure probability $\gamma\in(0,1)$,} there exists a constant $C > 0$ such that, with probability at least $1 - \gamma$:
\begin{align}\label{eq:calm-bound}
&\Delta_2^2(\hat{\tau}_{\CALM}, \tau^r) \leq \Delta_2^2(\cF, \tau^r) + C\frac{\log(1/\gamma)}{n^r} \notag\\
&+ C\biggl[\underbrace{\epsilon_{\mathrm{suff}}^2}_{\text{suff.}} + \underbrace{(L_\mu + L_\delta)^2 r_\phi^2}_{\text{align.}} + \underbrace{\textstyle\sum_a \Rad_{n^o}^2(\cM^o_a)}_{\text{OS}} \notag\\
&+ \underbrace{\textstyle\sum_a \Rad_{n^r}^2(\cD_a)}_{\text{calib.}} + \underbrace{\Rad_{n^r}^2(\cF)}_{\text{CATE}} \biggr], 
\end{align}
where $L_\mu,L_\delta$ are the Lipschitz constants linking alignment error to prediction error.
\end{theorem}

{\textbf{Interpretation.} 
The proof is in Appendix~\ref{app:proofs}. The first term is the approximation error of the CATE class, and $C\log(1/\gamma)/n^r$ is the finite-sample concentration cost. Inside the brackets, $\epsilon_{\mathrm{suff}}^2$ measures information lost by the embedding, $(L_\mu+L_\delta)^2r_\phi^2$ is the prediction error induced by imperfect alignment, the two summed Rademacher terms are the OS outcome-model and RCT calibration complexities, and $\Rad_{n^r}^2(\cF)$ is the complexity of the CATE regression.}

\subsection{Comparison with Imputation-Based Borrowing}

\begin{corollary}[When \CALM{} improves over \MROSCAR]\label{cor:calm-vs-mroscar}
Assume both estimators use function classes of comparable complexity for calibration \citep[see][for the \MROSCAR{} bound]{pal2026mismatch}. {Let $r_{\mathrm{im}}^2$ denote the oracle risk of imputing $\bV$ from $\bZ$, and let $\cG$ denote the imputation function class.}
Then \CALM{} yields a tighter bound than \MROSCAR{} whenever
\(
\epsilon_{\mathrm{suff}}^2 + (L_\mu + L_\delta)^2 r_\phi^2 + \textstyle\sum_a \Rad_{n^o}^2(\cM^o_a) 
< L^2 r_{\mathrm{im}}^2 + \textstyle\sum_a \Rad_{n^o}^2(\cM^{o,\mathrm{im}}_a) + \Rad_{n^o}^2(\cG).
\)
This holds when:
\begin{enumerate}[leftmargin=16pt, itemsep=2pt]
\item $d \ll p_o$, so outcome models in $\R^d$ have lower Rademacher complexity than those in $\R^{p_o}$;
\item $\bV$ is hard to reconstruct but easy to summarize: $r_\phi^2 \ll r_{\mathrm{im}}^2$ because the outcome-relevant information in $\bV$ lies on a low-dimensional manifold;
\item The sufficiency gap $\epsilon_{\mathrm{suff}}^2$ is small, i.e., the embedding captures most outcome-relevant variation.
\end{enumerate}
\end{corollary}

{The bound exposes the trade-off}: \CALM{} gains from dimensionality reduction but pays for information loss ($\epsilon_{\mathrm{suff}}^2$) and alignment imperfection ($r_\phi^2$), while \MROSCAR{} avoids information loss but pays the full imputation cost ($r_{\mathrm{im}}^2$ in $\R^{p_v}$).
{\SROSCAR{} avoids imputation as well, but discards all information outside the shared block; it is therefore strongest when $\bZ$ already captures the outcome-relevant variation.}

\subsection{Negative-Transfer Protection}

\begin{proposition}[Safe borrowing]\label{prop:safe}
Under the conditions of Theorem~\ref{thm:calm-main}, the pseudo-outcome $\hat{\psi}^r$ remains unbiased for $\tau^r(\bX^r)$ regardless of augmentation quality (Eq.~\ref{eq:pseudo-outcome}).
{Thus, both \CALM-Lin{} and \CALM-NN{} preserve the RCT CATE target: poor augmentation can affect efficiency, but it does not change the estimand. When alignment error $(L_\mu + L_\delta)^2 r_\phi^2$ is large, the borrowed augmentation may provide little variance reduction, and Stage~4 must rely more heavily on RCT data through the correction $\hat{\delta}(\bX^r)$. The practical difference is finite-sample stability: \CALM-Lin{} has a more controlled fallback, while \CALM-NN{} can suffer higher variance or approximation error under severe distributional shift if the learned encoder aligns poorly.}
\end{proposition}

\subsection{Specialization to Sparse Linear Models}\label{sec:linear}

We specialize \CALM{} to a linear setting.
Let the encoders be linear projections $\phi^o(\bX^o) = \bW^o (\bX^o)^\top$ and $\phi^r(\bX^r) = \bW^r (\bX^r)^\top$, where $\bW^o \in \R^{d \times p_o}$ and $\bW^r \in \R^{d \times p_r}$.
Let the outcome be linear in the embedding: $\tilde{\mu}^o_a(\bh) = \bm{\beta}_a^\top \bh$.

\begin{proposition}[Linear embedding \& imputation]\label{prop:linear-equiv}
Under Gaussian covariates with $\bV = \bm{\Lambda} \bZ + \bm{\epsilon}_V$ and $\bU \ind \bV \mid \bZ$, the optimal linear embedding that minimizes alignment error subject to preserving outcome-relevant information satisfies
\[
\bW^r_{\mathrm{opt}} = \bW^o
\begin{pmatrix}
\bm{0}_{p_z \times p_u} & \bm{I}_{p_z} \\
\bm{0}_{p_v \times p_u} & \bm{\Lambda}
\end{pmatrix},
\]
which is equivalent to imputing $\widehat{\bV} = \bm{\Lambda}\bZ$ and then projecting $(\bZ, \widehat{\bV})$ through $\bW^o$.
\end{proposition}

{In the linear case,} the alignment-based approach with a $d$-dimensional embedding is equivalent to imputation followed by dimensionality reduction to $\R^d$.
{Thus, the embedding approach contains the linear imputation as a special case: when $d = p_o$, it recovers \MROSCAR; when $d < p_o$, it can gain additional variance reduction through projection.}

\begin{corollary}[\CALM-Lin{} risk bound]\label{cor:linear-calm}
Under the conditions of Proposition~\ref{prop:linear-equiv}, if the outcome model has sparsity $s$ in the embedding space and the discrepancy has sparsity $s_\delta$, then
\begin{align*} 
\Delta_2^2(\hat{\tau}_{\CALM}, \tau^r) &\lesssim \underbrace{\Delta_2^2(\cF, \tau^r)}_{\text{approx.}} + \underbrace{(L_\mu+L_\delta)^2 \mathrm{tr}(\bm{\Sigma}_{\bV \mid \bZ})}_{\text{irred.\ noise}} \\
&\quad + \underbrace{\tfrac{s \log d}{n^o}}_{\text{OS}} + \underbrace{\tfrac{s_\delta \log d}{n^r}}_{\text{calib.}} + \underbrace{\tfrac{s_\tau \log p_r}{n^r}}_{\text{CATE}},
\end{align*}
where $s_\tau$ is the CATE sparsity.
Compared to \MROSCAR, the outcome model and calibration complexities scale with $\log d$ instead of $\log p_o$, yielding tighter rates when $d \ll p_o$.
\end{corollary}

\vspace{-0.5em}
\section{Practical Implementation}\label{sec:implementation}

In the linear instantiation (\CALM-Lin), {each stage reduces to penalized regression}: Stage~1 fits a reduced-rank regression (or LASSO in a PCA-reduced space) of $Y$ on $\bX^o$ in the OS; {Stage~2 constructs $\bW^r=\bW^o\hat{\bm M}$, where $\hat{\bm M}$ maps $(\bU,\bZ)$ to $(\bZ,\hat{\bm{\Lambda}}\bZ)$ and $\hat{\bm{\Lambda}}$ estimates $\E[\bV \mid \bZ]$ by ridge/LASSO,} and calibrates via LASSO on $\bX^r$; Stages~3--4 follow the standard \ROSCAR{} pipeline.
In the neural instantiation (\CALM-NN), the encoders $\phi^o$ and $\phi^r$ are MLPs with separate parameters and per-arm outcome heads; our experiments use the conditional-mean version of the $\bZ$-conditioned alignment objective in Appendix~\ref{app:alignment-objectives}, while MMD~\eqref{eq:mmd-align} is a distribution-matching alternative.
Stage~1 {uses} OS data alone, Stage~2 jointly optimizes $\phi^r$ and $\hat{\delta}^t_a$ with the alignment penalty ({optionally initialized from} $\phi^o$ restricted to shared features), Stage~4 uses LASSO on $\bX^r$ for interpretability, and all nuisance estimates are cross-fitted over $K$ RCT folds.
The embedding dimension $d$ controls the bias--variance trade-off: for \CALM-Lin, $d$ is set {by} cross-validated OS prediction error; for \CALM-NN, $d$ is {selected by} cross-validated CATE calibration error on the RCT.

\textbf{Choosing between \CALM-Lin and \CALM-NN.}\ Prefer \CALM-Lin when the CATE is plausibly smooth or close to linear in $\bX^r$, or when the trial is small ($n^r \lesssim 200$): there the safety guarantee of \Cref{prop:safe} applies in full, and the linear instantiation has the lowest variance among the calibration-family methods. Prefer \CALM-NN when cross-validation shows that a nonlinear OS outcome model clearly outperforms a linear one and $n^o$ is large enough to learn a useful encoder; the GPS and STAR analyses of \Cref{sec:realdata} both fall in this regime. In borderline cases, we compare the cross-fitted Stage~2 calibration loss and final pseudo-outcome regression error on held-out RCT folds. If the nonlinear encoder does not improve these diagnostics, the linear version is the default because it borrows less aggressively and is easier to audit.

\section{Experiments}\label{sec:experiments}

{We evaluate} \CALM{}\footnote{Code to reproduce all experiments is available at \url{https://github.com/AsiaeeLab/calm-cate}.} {with simulations that target} the predictions of Section~\ref{sec:theory}.

\subsection{Simulation Design}

\textbf{Data-generating process.}\ 
We generate (OS, RCT) pair {as follows}.
Let $p_z = 30$, $p_u = 10$, $p_v = 20$, and $d_{\mathrm{true}} = 5$ denote the intrinsic dimension of the outcome-relevant signal.
Shared covariates are drawn as $\bZ \sim \mathcal{N}(\bm{0}, \bm{\Sigma}_{\bZ\bZ})$ with AR(1) correlation ($\rho = 0.5$);
(i)~OS-only covariates $\bV = \bm{\Lambda}_{\bV\bZ} \bZ + \bm{\epsilon}_V$, $\bm{\epsilon}_V \sim \mathcal{N}(\bm{0}, \sigma_V^2 \bm{I})$, and
(ii)~RCT-only covariates $\bU = \bm{\Lambda}_{\bU\bZ} \bZ + \bm{\epsilon}_U$, $\bm{\epsilon}_U \sim \mathcal{N}(\bm{0}, \sigma_U^2 \bm{I})$, are generated conditionally on $\bZ$.
Outcomes follow $Y^s(a) = f_a(\bm{P}^s \bX^\top) + \delta^s_a(\bZ) + \epsilon$, where $\bm{P}^s$ projects to a $d_{\mathrm{true}}$-dimensional subspace, $f_a$ is a nonlinear function (with linear and sinusoidal variants), and $\delta^s_a$ captures population-specific shifts.
Treatment assignment in the OS uses a logistic model on 10 covariates; the RCT uses $\pi^r_{+1} = 0.5$.

\textbf{Factors varied.}\ 
Each experiment varies one factor while holding others at default values ($n^r = 500$, $\sigma_V^2 = 1.0$, $d_{\mathrm{true}} = 5$, linear outcome, shift magnitude $0.5$).
{We vary six factors}:
imputation difficulty, $\sigma_V^2 \in \{0.1, 0.25, 0.5, 1.0, 2.0\}$ (higher $\sigma_V^2$ makes $\bV$ harder to predict from $\bZ$);
intrinsic dimension, $d_{\mathrm{true}} \in \{2, 3, 5, 10, 15, 20\}$;
RCT sample size, $n^r \in \{100, 250, 500, 1{,}000, 2{,}000\}$ with $n^o = 10{,}000$ fixed;
outcome nonlinearity (linear, quadratic, sinusoidal);
outcome shift magnitude, $\|\delta^r_a - \delta^o_a\| \in \{0, 0.25, 0.5, 1.0, 2.0, 5.0\}$;
and shared covariate proportion, $p_z / (p_z + p_v) \in \{0.3, 0.5, 0.7, 0.9\}$.

\subsection{Methods Under Comparison}

We compare eight methods:
\textbf{Naive}, RCT-only CATE estimation with no augmentation ($m = 0$);
\textbf{\RACER}, RCT-only augmentation using outcome models fitted on $\bX^r$ in the RCT;
\textbf{\SROSCAR} \citep{pal2026mismatch}, which borrows from OS using only shared covariates $\bZ$;
\textbf{\MROSCAR} \citep{pal2026mismatch}, imputation-based borrowing;
\textbf{\CALM-Lin}, the linear embedding version of \CALM{};
\textbf{\CALM-NN}, the neural network version of \CALM{} (residual MLP encoders with an alignment loss);
and \textbf{HTCE-T} and \textbf{HTCE-DR}, the transfer T-learner and DR-learner of \citet{bica2022transfer}, respectively.
For HTCE-T and HTCE-DR, we use the shared/private encoder architectures of \citet{bica2022transfer} and {compute predictions by cross-fitting over} RCT folds to avoid outcome leakage.

\subsection{Evaluation Metrics}

We report the RMSE of CATE, {\small$\bigl[(1/n^r) \sum_{i} (\hat{\tau}(\bX^r_i) - \tau^r(\bX^r_i))^2\bigr]^{1/2}$}, averaged over 20 replicates.

\subsection{Results: Linear CATE Regime}
\label{sec:results-baseline}

{Unless noted, reported values are} mean RMSE over 20 replicates.
The baseline regime comprises 29 settings in which the CATE is linear in the target covariates and a single mismatch factor is varied at a time.
Across these 29 settings, the lowest mean RMSE is achieved by \RACER{} in 11, \CALM-Lin{} in 9, \MROSCAR{} in 6, \SROSCAR{} in 2, and HTCE-T in 1.
All four calibration-based methods (\RACER{}, \SROSCAR{}, \MROSCAR{}, \CALM-Lin{}) are effectively indistinguishable in this regime: pairwise differences in mean RMSE are below $10^{-3}$, and which method achieves the argmin depends on the factor varied.
Table~\ref{tab:main-results} reports representative values.

\textbf{Imputation difficulty (Figure~\ref{fig:rmse-sigma}).}\ 
As $\sigma_V^2$ increases, all calibration-based methods degrade together.
The argmin shifts with the noise level: \RACER{} is best at $\sigma_V^2 \in \{0.1, 0.5, 2.0\}$ (RMSE $0.76$, $0.91$, and $1.45$), \CALM-Lin{} at $\sigma_V^2 = 0.25$ ($0.83$), and \MROSCAR{} at $\sigma_V^2 = 1.0$ ($1.03$).
\CALM-NN {has higher RMSE than} the calibration-based methods in this linear-CATE setting.

{Additional linear-regime sweeps including intrinsic dimension, shared-covariate proportion, RCT sample size, outcome-model nonlinearity, and negative transfer appear in Appendix~\ref{app:experiments}.}

\begin{figure*}[t]
\centering 
\begin{subfigure}[t]{0.33\textwidth}
  \centering
  \includegraphics[width=\linewidth]{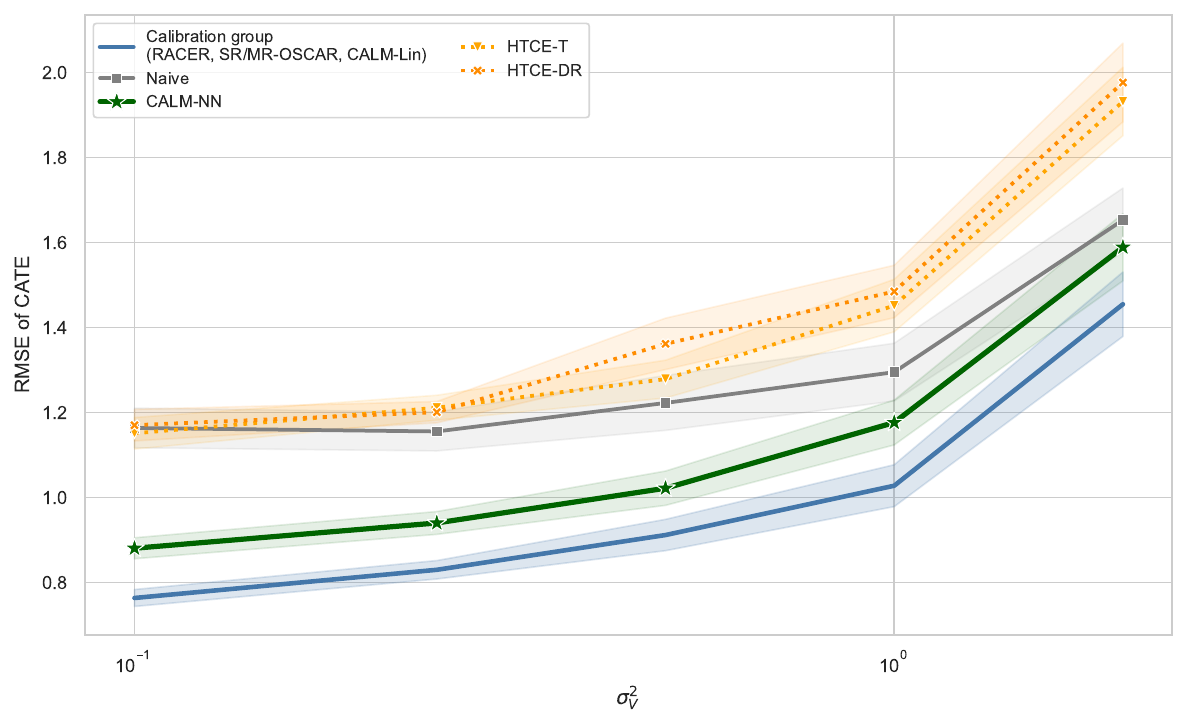}
  \caption{Imputation difficulty ($\sigma_V^2$).}
  \label{fig:rmse-sigma}
\end{subfigure}\hfill
\begin{subfigure}[t]{0.33\textwidth}
  \centering
  \includegraphics[width=\linewidth]{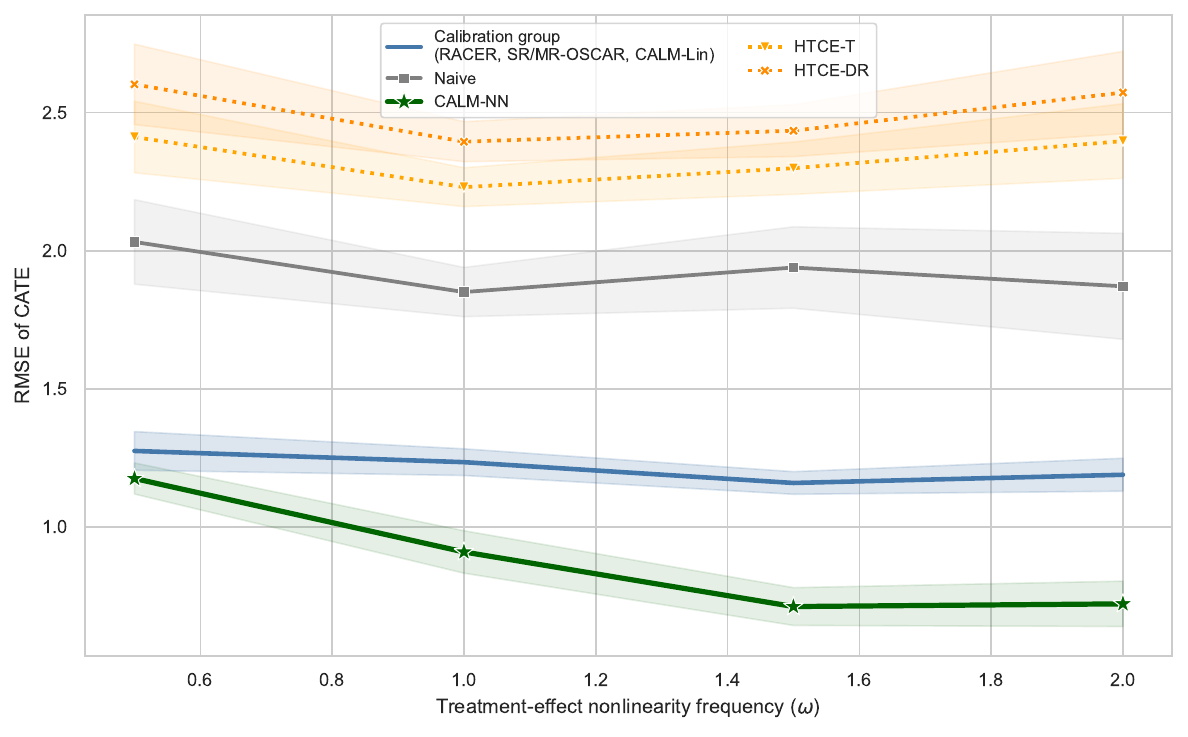}
  \caption{CATE nonlinearity ($\omega$).}
  \label{fig:tau-nonlinearity}
\end{subfigure}\hfill
\begin{subfigure}[t]{0.33\textwidth}
  \centering
  \includegraphics[width=\linewidth]{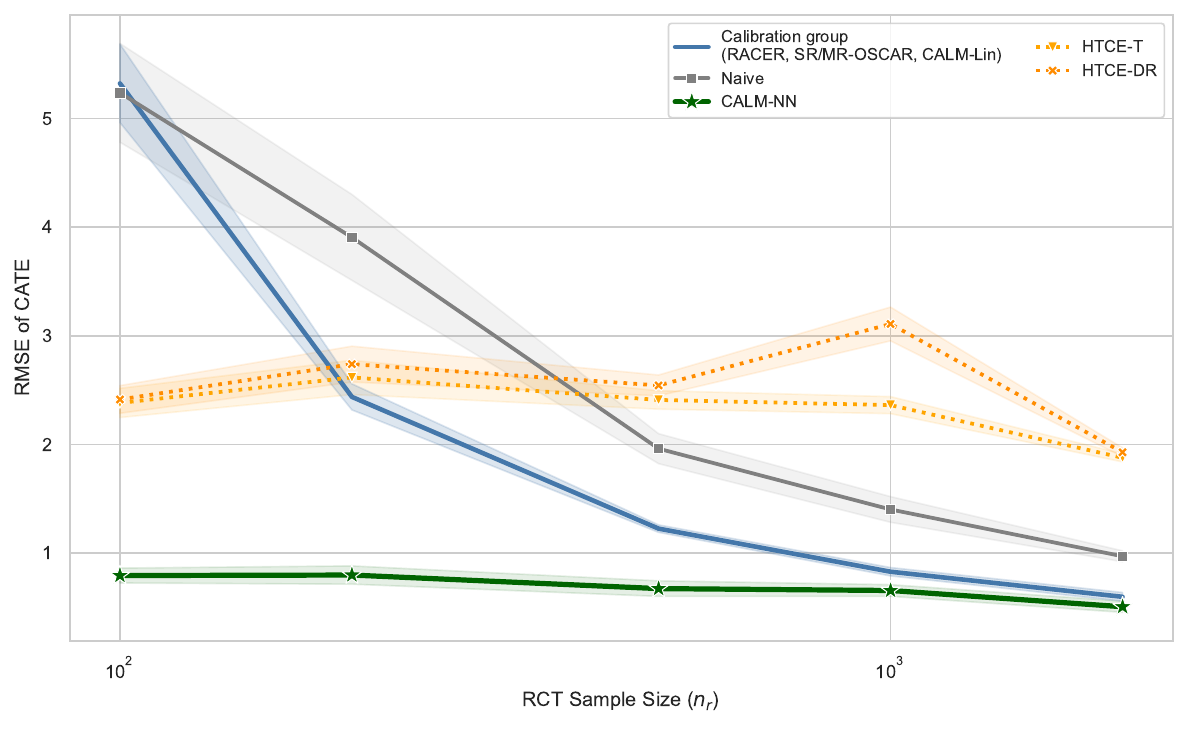}
  \caption{RCT sample size ($n^r$).}
  \label{fig:tier4-sample-size}
\end{subfigure}
\caption{RMSE of CATE estimation across three experimental sweeps. Mean over 20 replicates. In all panels, the blue band groups four calibration-based methods (\RACER, \SROSCAR, \MROSCAR, \CALM-Lin) whose RMSEs are nearly identical; the band spans their min--max envelopes ($\text{mean} \pm \text{SE}$). Individual lines show the remaining methods.
\textbf{(a)}~$n^r = 500$, $n^o = 10{,}000$, $d_{\mathrm{true}} = 5$, linear outcome.
\textbf{(b)}~Shared-latent DGP where $\bX^r$ carries information about $\bV$ beyond $\bZ$; \CALM-NN separates from the calibration group as $\omega$ increases.
\textbf{(c)}~Nonlinear-CATE regime: \CALM-NN maintains RMSE below $0.8$ even at $n^r = 100$, where calibration-based methods exceed $5.3$.}
\label{fig:main-results}
\vspace{-1.5em}
\end{figure*}

\begin{table}[t]
\centering
\caption{Mean RMSE across imputation difficulty ($\sigma_V^2$, linear outcome) and outcome nonlinearity ($\sigma_V^2 = 1.0$), averaged over 20 replicates. \textbf{Bold}: lowest per column.}
\label{tab:main-results}
\resizebox{\columnwidth}{!}{\footnotesize\setlength{\tabcolsep}{3pt}
\begin{tabular}{@{}lcccccc@{}}
\toprule
& \multicolumn{4}{c}{Imputation difficulty ($\sigma_V^2$)} & \multicolumn{2}{c}{Outcome} \\
\cmidrule(lr){2-5} \cmidrule(lr){6-7}
Method & $0.1$ & $0.5$ & $1.0$ & $2.0$ & Quad. & Sin. \\
\midrule
Naive       & 1.16 & 1.22 & 1.30 & 1.65 & 1.70 & 1.29 \\
Calib.\ group$^\dagger$ & \textbf{0.76} & \textbf{0.91} & \textbf{1.03} & \textbf{1.45} & \textbf{1.45} & \textbf{1.06} \\
\CALM-NN  & 0.88 & 1.02 & 1.18 & 1.59 & 1.54 & 1.15 \\
HTCE-T      & 1.15 & 1.28 & 1.45 & 1.93 & 1.81 & 1.43 \\
HTCE-DR     & 1.17 & 1.36 & 1.48 & 1.98 & 1.92 & 1.50 \\
\multicolumn{7}{@{}l@{}}{\scriptsize $^\dagger$\RACER{}, \SROSCAR{}, \MROSCAR{}, \CALM-Lin: single representative; pairwise $\Delta$RMSE${<}10^{-3}$.} \\
\bottomrule
\end{tabular}}
\end{table}

\subsection{Results: Nonlinear-CATE Regime}
\label{sec:results-nonlinear}

We turn to a regime in which the CATE is nonlinear in the target covariates.
We modify the DGP so that the outcome-relevant signal is mediated entirely by $\bV$ through shared latent factors: $\bU$ and $\bV$ share a 5-dimensional latent factor $\bm{H}$ (scale $2.0$ in both measurement models, $\sigma_V^2 = \sigma_U^2 = 0.1$), and the outcome depends only on $\bV$ (signal weights $w_Z = w_U = 0$, $w_V = 2$).
The CATE takes a nonlinear sinusoidal form with frequency parameter $\omega$, so that $\tau^r(\bX^r)$ is nonlinear after marginalization over $\bV$.
This regime, comprising 22 settings across the sweeps described below, is not included in the baseline grid.
\CALM-NN attains the lowest mean RMSE in all 22 settings, with margins ranging from $0.09$ to $4.53$.

\textbf{Treatment-effect frequency (Figure~\ref{fig:tau-nonlinearity}).}\ 
Across $\omega \in \{0.5, 1.0, 1.5, 2.0\}$, \CALM-NN attains the lowest RMSE in all four settings.
At $\omega = 1.5$, \CALM-NN achieves RMSE $0.71$ compared with $1.16$ for \CALM-Lin; at $\omega = 2.0$, RMSE is $0.72$ versus $1.19$ for the best calibration-based method, a $39\%$ reduction.
{This matches} the regime predicted by Corollary~\ref{cor:calm-vs-mroscar}: when outcome-relevant information is low-dimensional and the CATE is nonlinear, learned embeddings outperform imputation-based approaches.

\textbf{Sample efficiency (Figure~\ref{fig:tier4-sample-size}).}\ 
We vary $n^r \in \{100, 200, 500, 1000, 2000\}$ within the nonlinear-CATE DGP ($\omega = 1.5$, $w_Z = 0$).
\CALM-NN maintains RMSE in the range $0.51$--$0.79$ across all sample sizes, whereas calibration-based methods degrade sharply at small $n^r$: RMSE $5.32$ at $n^r = 100$ compared with $0.79$ for \CALM-NN.
Even at $n^r = 2{,}000$, \CALM-NN retains a modest advantage ($0.51$ vs.\ $0.60$).
{This stability reflects the use of} $10{,}000$ OS samples for representation learning, making the quality of the learned features largely independent of $n^r$.

{Further nonlinear-regime experiments including robustness to the shared-covariate signal weight, the CATE functional form, and the latent coupling strength are in Appendix~\ref{app:experiments}.}

\vspace{-0.6em}

\subsection{Real-Data Evaluation}\label{sec:realdata}
We complement the simulations with two analyses on real data, where the DGP is outside our control.

\textbf{Greenlight Plus (GPS) with an EHR cohort.}\ 
Following the covariate-mismatch protocol of \citet{pal2026mismatch}, we link the Greenlight Plus trial~\citep{heerman2022greenlight} subset ($n^r = 330$) to EHR controls ($n^o = 8{,}867$) and mask the EHR-only insurance covariate from the trial as $\bV$. The outcome is 24-month weight-for-length $z$-score; we report mean $95\%$ bootstrap-CI width ($B = 200$). \Cref{tab:gps} shows that only \CALM-NN materially tightens intervals over trial-only RACER ($17.5\%$), while calibration variants stay within $1\%$ and transfer-only neural baselines widen them.

\begin{table}[t]
\centering
\caption{GPS real-data evaluation: mean width of the per-unit $95\%$ bootstrap CI ($B = 200$; $n^r = 330$, $n^o = 8{,}867$). Only \CALM-NN improves materially over the trial-only RACER baseline.}
\label{tab:gps}
{\footnotesize
\begin{tabular}{lcc}
\toprule
Method & Mean $95\%$ CI width & vs.\ RACER \\
\midrule
Naive & $2.25$ & $-13.2\%$ \\
RACER & $1.98$ & --- \\
SR-OSCAR & $1.98$ & $0.2\%$ \\
MR-OSCAR & $1.97$ & $0.7\%$ \\
CALM-Lin & $1.98$ & $0.0\%$ \\
\textbf{CALM-NN} & $\mathbf{1.64}$ & $\mathbf{17.5\%}$ \\
HTCE-T & $3.49$ & $-75.9\%$ \\
HTCE-DR & $4.12$ & $-107.8\%$ \\
\bottomrule
\end{tabular}}
\end{table}

\textbf{Semi-real Tennessee STAR.}\ 
Following \citet{kallus2018removing}, we build the OS by outcome-dependent subsampling of STAR students and mask kindergarten free-lunch from the trial side. Because all units come from STAR, we estimate per-unit ground-truth CATEs with a cross-fitted random-forest T-learner on the full first-grade cohort and report RMSE over $15$ replicates at each trial fraction $q$. \Cref{tab:star} shows that \CALM-NN leads at $q = 0.25$ and $q = 1.0$, with ties in between.

\begin{table}[t]
\centering
\caption{Semi-real STAR: RMSE against the ground-truth CATE ($15$ replicates) at trial fraction $q$. \CALM-NN leads at the extremes and ties in between. Naive ($\approx 1064$) and the HTCE baselines ($\approx 420$--$465$) are omitted for readability.}
\label{tab:star}
{\footnotesize
\begin{tabular}{lccccc}
\toprule
$q$ & RACER & SR-OSC. & MR-OSC. & C-Lin & C-NN \\
\midrule
$0.25$ & $52.12$ & $51.20$ & $51.36$ & $51.42$ & $\mathbf{50.61}$ \\
$0.50$ & $\mathbf{31.33}$ & $31.40$ & $31.65$ & $\mathbf{31.33}$ & $31.53$ \\
$0.75$ & $\mathbf{25.87}$ & $25.97$ & $25.89$ & $25.90$ & $26.04$ \\
$1.00$ & $23.60$ & $23.74$ & $23.54$ & $23.60$ & $\mathbf{22.30}$ \\
\bottomrule
\end{tabular}}
\end{table}


\vspace{-0.6em}
\section{Conclusion}\label{sec:conclusion}
\vspace{-0.6em}

{
\CALM{} shows that CATE estimation under covariate mismatch can borrow through an outcome-relevant aligned representation, rather than through full reconstruction of missing covariates.
The risk bounds make this trade-off explicit by separating alignment error, embedding sufficiency, nuisance-model complexity, and final CATE-regression complexity.
The empirical results mirror this pattern: linear calibration is stable when the CATE is simple, while \CALM-NN gains when nonlinear OS-learned features improve held-out trial calibration.
Because alignment quality is not directly observable in the RCT, the method is most useful when held-out calibration and residual diagnostics indicate that the learned representation carries outcome-relevant OS signal without overwhelming the trial.
The method is therefore best viewed as calibrated borrowing, not a replacement for randomized identification. 
}

\begin{acknowledgements}
This work was supported in part by the Patient-Centered Outcomes Research Institute (PCORI) under award ME-2023C1-32148. A.A. was partly supported by the National Institute of Mental Health under award R01MH139379. We thank Jared Huling for his support during the early development of this work.
\end{acknowledgements}


\bibliography{references}

@article{asiaee2025improving,
  title = {Improving Precision of {RCT}-Based {CATE} Estimation using Data Borrowing with Double Calibration},
  author = {Asiaee, Amir and Di Gravio, Chiara and Beck, Cole and Mei, Yuting and Pal, Samhita and Huling, Jared D.},
  journal={arXiv preprint arXiv:2306.17478},
  year = {2025},
  eprint = {2306.17478},
  archivePrefix = {arXiv},
}

@article{pal2026mismatch,
  title={Improving {RCT}-Based {CATE} Estimation Under Covariate Mismatch via Double Calibration},
  author={Pal, Samhita and Huling, Jared D. and Asiaee, Amir},
  journal={arXiv preprint arXiv:2603.17066},
  year={2026},
  eprint={2603.17066},
  archivePrefix={arXiv}
}

@inproceedings{bica2022transfer,
  title={Transfer Learning on Heterogeneous Feature Spaces for Treatment Effects Estimation},
  author={Bica, Ioana and van der Schaar, Mihaela},
  booktitle={Advances in Neural Information Processing Systems},
  volume={35},
  pages={37184--37198},
  year={2022}
}

@article{colnet2024causal,
  title={Causal inference methods for combining randomized trials and observational studies: a review},
  author={Colnet, B{\'e}n{\'e}dicte and Mayer, Imke and Chen, Guanhua and Dieng, Awa and Li, Ruohong and Varoquaux, Ga{\"e}l and Vert, Jean-Philippe and Josse, Julie and Yang, Shu},
  journal={Statistical Science},
  volume={39},
  number={1},
  pages={165--191},
  year={2024}
}

@article{degtiar2023review,
  title={A review of generalizability and transportability},
  author={Degtiar, Irina and Rose, Sherri},
  journal={Annual Review of Statistics and Its Application},
  volume={10},
  pages={501--524},
  year={2023}
}

@article{kennedy2023towards,
  title={Towards optimal doubly robust estimation of heterogeneous causal effects},
  author={Kennedy, Edward H.},
  journal={Electronic Journal of Statistics},
  volume={17},
  number={2},
  pages={3008--3049},
  year={2023}
}

@article{nie2021quasi,
  title={Quasi-oracle estimation of heterogeneous treatment effects},
  author={Nie, Xinkun and Wager, Stefan},
  journal={Biometrika},
  volume={108},
  number={2},
  pages={299--319},
  year={2021}
}

@article{kunzel2019metalearners,
  title={Metalearners for estimating heterogeneous treatment effects using machine learning},
  author={K{\"u}nzel, S{\"o}ren R. and Sekhon, Jasjeet S. and Bickel, Peter J. and Yu, Bin},
  journal={Proceedings of the National Academy of Sciences},
  volume={116},
  number={10},
  pages={4156--4165},
  year={2019}
}

@article{athey_generalized_2019,
  title={Generalized random forests},
  author={Athey, Susan and Tibshirani, Julie and Wager, Stefan},
  journal={The Annals of Statistics},
  volume={47},
  number={2},
  pages={1148--1178},
  year={2019}
}

@article{cheng2021adaptive,
  title={Adaptive Combination of Randomized and Observational Data},
  author={Cheng, David and Cai, Tianxi},
  journal={arXiv preprint arXiv:2111.15012},
  year={2021},
  eprint={2111.15012},
  archivePrefix={arXiv}
}

@article{oberst2022understanding,
  title={Understanding the Risks and Rewards of Combining Unbiased and Possibly Biased Estimators, with Applications to Causal Inference},
  author={Oberst, Michael and D'Amour, Alexander and Chen, Minmin and Wang, Yuyan and Sontag, David and Yadlowsky, Steve},
  journal={arXiv preprint arXiv:2205.10467},
  year={2022}
}

@inproceedings{johansson2016learning,
  title={Learning representations for counterfactual inference},
  author={Johansson, Fredrik and Shalit, Uri and Sontag, David},
  booktitle={International Conference on Machine Learning},
  pages={3020--3029},
  year={2016}
}

@inproceedings{shalit2017estimating,
  title={Estimating individual treatment effect: generalization bounds and algorithms},
  author={Shalit, Uri and Johansson, Fredrik D. and Sontag, David},
  booktitle={International Conference on Machine Learning},
  pages={3076--3085},
  year={2017}
}

@inproceedings{yao2018representation,
  title={Representation learning for treatment effect estimation from observational data},
  author={Yao, Liuyi and Li, Sheng and Li, Yaliang and Huai, Mengdi and Gao, Jing and Zhang, Aidong},
  booktitle={Advances in Neural Information Processing Systems},
  volume={31},
  year={2018}
}

@inproceedings{shi2019adapting,
  title={Adapting neural networks for the estimation of treatment effects},
  author={Shi, Claudia and Blei, David M. and Veitch, Victor},
  booktitle={Advances in Neural Information Processing Systems},
  volume={32},
  year={2019}
}

@article{ben2010theory,
  title={A theory of learning from different domains},
  author={Ben-David, Shai and Blitzer, John and Crammer, Koby and Kulesza, Alex and Pereira, Fernando and Vaughan, Jennifer Wortman},
  journal={Machine Learning},
  volume={79},
  number={1--2},
  pages={151--175},
  year={2010}
}

@article{gretton2012kernel,
  title={A kernel two-sample test},
  author={Gretton, Arthur and Borgwardt, Karsten M. and Rasch, Malte J. and Sch{\"o}lkopf, Bernhard and Smola, Alexander},
  journal={Journal of Machine Learning Research},
  volume={13},
  pages={723--773},
  year={2012}
}

@article{bartlett2006local,
  title={Empirical minimization},
  author={Bartlett, Peter L. and Mendelson, Shahar},
  journal={Probability Theory and Related Fields},
  volume={135},
  number={3},
  pages={311--334},
  year={2006}
}

@book{wainwright2019high,
  title={High-Dimensional Statistics: A Non-Asymptotic Viewpoint},
  author={Wainwright, Martin J.},
  publisher={Cambridge University Press},
  year={2019}
}

@article{ganin2016domain,
  title={Domain-adversarial training of neural networks},
  author={Ganin, Yaroslav and Ustinova, Evgeniya and Ajakan, Hana and Germain, Pascal and Larochelle, Hugo and Laviolette, Fran{\c{c}}ois and Marchand, Mario and Lempitsky, Victor},
  journal={Journal of Machine Learning Research},
  volume={17},
  number={59},
  pages={1--35},
  year={2016}
}

@inproceedings{long2015learning,
  title={Learning Transferable Features with Deep Adaptation Networks},
  author={Long, Mingsheng and Cao, Yue and Wang, Jianmin and Jordan, Michael I.},
  booktitle={Proceedings of the 32nd International Conference on Machine Learning},
  series={JMLR Workshop and Conference Proceedings},
  volume={37},
  pages={97--105},
  publisher={JMLR.org},
  year={2015}
}

@article{pan2010survey,
  title={A survey on transfer learning},
  author={Pan, Sinno Jialin and Yang, Qiang},
  journal={IEEE Transactions on Knowledge and Data Engineering},
  volume={22},
  number={10},
  pages={1345--1359},
  year={2010}
}

@article{weiss2016survey,
  title={A survey of transfer learning},
  author={Weiss, Karl and Khoshgoftaar, Taghi M. and Wang, DingDing},
  journal={Journal of Big Data},
  volume={3},
  number={1},
  pages={9},
  year={2016},
  doi={10.1186/s40537-016-0043-6}
}

@article{hardoon2004canonical,
  title={Canonical correlation analysis: An overview with application to learning methods},
  author={Hardoon, David R. and Szedmak, Sandor and Shawe-Taylor, John},
  journal={Neural Computation},
  volume={16},
  number={12},
  pages={2639--2664},
  year={2004}
}

@inproceedings{andrew2013deep,
  title={Deep canonical correlation analysis},
  author={Andrew, Galen and Arora, Raman and Bilmes, Jeff and Livescu, Karen},
  booktitle={International Conference on Machine Learning},
  pages={1247--1255},
  year={2013}
}

@book{li2018sufficient,
  title={Sufficient dimension reduction: Methods and applications with {R}},
  author={Li, Bing},
  publisher={CRC Press},
  year={2018}
}

@article{bousmalis2016domain,
  title={Domain separation networks},
  author={Bousmalis, Konstantinos and Trigeorgis, George and Silberman, Nathan and Krishnan, Dilip and Erhan, Dumitru},
  journal={Advances in Neural Information Processing Systems},
  volume={29},
  year={2016}
}

@article{kosorok2019precision,
  title={Precision medicine},
  author={Kosorok, Michael R. and Laber, Eric B.},
  journal={Annual Review of Statistics and Its Application},
  volume={6},
  pages={263--286},
  year={2019}
}

@book{sugiyama2012machine,
  title={Machine Learning in Non-Stationary Environments: Introduction to Covariate Shift Adaptation},
  author={Sugiyama, Masashi and Kawanabe, Motoaki},
  publisher={MIT Press},
  year={2012}
}

@book{Rassler2002,
  title={Statistical Matching: A Frequentist Theory, Practical Applications, and Alternative {Bayesian} Approaches},
  author={R{\"a}ssler, Susanne},
  publisher={Springer},
  year={2002}
}

@article{dahabreh2020extending,
  title={Extending inferences from a randomized trial to a new target population},
  author={Dahabreh, Issa J. and Robertson, Sarah E. and Steingrimsson, Jon A. and Stuart, Elizabeth A. and Hern{\'a}n, Miguel A.},
  journal={Statistics in Medicine},
  volume={39},
  number={14},
  pages={1999--2014},
  year={2020},
  doi={10.1002/sim.8426}
}

@article{karlsson2025robust,
  title={Robust Estimation of Heterogeneous Treatment Effects in Randomized Trials Leveraging External Data},
  author={Karlsson, Rickard and De Bartolomeis, Piersilvio and Dahabreh, Issa J. and Krijthe, Jesse H.},
  journal={arXiv preprint arXiv:2507.03681},
  year={2025}
}

@article{wang2007statistics,
  title={Statistics in medicine---reporting of subgroup analyses in clinical trials},
  author={Wang, Rui and Lagakos, Stephen W. and Ware, James H. and Hunter, David J. and Drazen, Jeffrey M.},
  journal={New England Journal of Medicine},
  volume={357},
  number={21},
  pages={2189--2194},
  year={2007}
}

@article{kent2020path,
  title={The Predictive Approaches to Treatment effect Heterogeneity ({PATH}) Statement},
  author={Kent, David M. and Paulus, Jessica K. and van Klaveren, David and D'Agostino, Ralph and Goodman, Steve and Hayward, Rodney and Ioannidis, John P. A. and Patrick-Lake, Bray and Morton, Sally and Pencina, Michael and Raman, Gowri and Ross, Joseph S. and Selker, Harry P. and Varadhan, Ravi and Vickers, Andrew and Wong, John B. and Steyerberg, Ewout W.},
  journal={Annals of Internal Medicine},
  volume={172},
  number={1},
  pages={35--45},
  year={2020},
  doi={10.7326/M18-3667}
}

@article{hill2011bayesian,
  title={Bayesian Nonparametric Modeling for Causal Inference},
  author={Hill, Jennifer L.},
  journal={Journal of Computational and Graphical Statistics},
  volume={20},
  number={1},
  pages={217--240},
  year={2011},
  doi={10.1198/jcgs.2010.08162}
}

@article{heerman2022greenlight,
	title = {The {Greenlight} {Plus} {Trial}: {Comparative} effectiveness of a health information technology intervention vs. health communication intervention in primary care offices to prevent childhood obesity},
	volume = {123},
	doi = {10.1016/j.cct.2022.106987},
	journal = {Contemporary Clinical Trials},
	author = {Heerman, William J. and Perrin, Eliana M. and Yin, H. Shonna and Schildcrout, Jonathan S. and Delamater, Alan M. and Flower, Kori B. and Sanders, Lee and Wood, Charles and Kay, Melissa C. and Adams, Laura E. and Rothman, Russell L.},
	year = {2022},
	pages = {106987},
}

@article{hahn2020bayesian,
  title={Bayesian regression tree models for causal inference: regularization, confounding, and heterogeneous effects},
  author={Hahn, P. Richard and Murray, Jared S. and Carvalho, Carlos M.},
  journal={Bayesian Analysis},
  volume={15},
  number={3},
  pages={965--1056},
  year={2020}
}

@inproceedings{hassanpour2020learning,
  title={Learning disentangled representations for counterfactual regression},
  author={Hassanpour, Negar and Greiner, Russell},
  booktitle={International Conference on Learning Representations},
  year={2020}
}

@article{cook2007sufficient,
  title={Fisher lecture: Dimension reduction in regression},
  author={Cook, R. Dennis},
  journal={Statistical Science},
  volume={22},
  number={1},
  pages={1--26},
  year={2007}
}

@article{ma2012semiparametric,
  title={A semiparametric approach to dimension reduction},
  author={Ma, Yanyuan and Zhu, Liping},
  journal={Journal of the American Statistical Association},
  volume={107},
  number={497},
  pages={168--179},
  year={2012}
}

@article{colnet2021generalizing,
  title={Causal effect on a target population: A sensitivity analysis to handle missing covariates},
  author={Colnet, B{\'e}n{\'e}dicte and Josse, Julie and Varoquaux, Ga{\"e}l and Scornet, Erwan},
  journal={Journal of Causal Inference},
  volume={10},
  number={1},
  pages={372--414},
  year={2022},
  doi={10.1515/jci-2021-0059}
}

@article{ibrahim2000power,
  title={Power prior distributions for regression models},
  author={Chen, Ming-Hui and Ibrahim, Joseph G.},
  journal={Statistical Science},
  volume={15},
  number={1},
  pages={46--60},
  year={2000}
}

@article{hobbs2011hierarchical,
  title={Hierarchical commensurate and power prior models for adaptive incorporation of historical information in clinical trials},
  author={Hobbs, Brian P. and Carlin, Bradley P. and Mandrekar, Sumithra J. and Sargent, Daniel J.},
  journal={Biometrics},
  volume={67},
  number={3},
  pages={1047--1056},
  year={2011}
}

@inproceedings{rosenstein2005transfer,
  title={To transfer or not to transfer},
  author={Rosenstein, Michael T. and Marx, Zvika and Kaelbling, Leslie Pack and Dietterich, Thomas G.},
  booktitle={NIPS 2005 Workshop on Inductive Transfer},
  year={2005}
}

@article{ghosh2021sufficient,
  title={Sufficient Dimension Reduction for Feasible and Robust Estimation of Average Causal Effect},
  author={Ghosh, Trinetri and Ma, Yanyuan and de Luna, Xavier},
  journal={Statistica Sinica},
  volume={31},
  number={2},
  pages={821--842},
  year={2021},
  doi={10.5705/ss.202018.0416}
}

@article{li1991sir,
  title={Sliced Inverse Regression for Dimension Reduction},
  author={Li, Ker-Chau},
  journal={Journal of the American Statistical Association},
  volume={86},
  number={414},
  pages={316--327},
  year={1991},
  publisher={Taylor \& Francis},
  doi={10.1080/01621459.1991.10475035}
}

@article{fukumizu2009kernel,
  title={Kernel dimension reduction in regression},
  author={Fukumizu, Kenji and Bach, Francis R. and Jordan, Michael I.},
  journal={The Annals of Statistics},
  volume={37},
  number={4},
  pages={1871--1905},
  year={2009},
  publisher={Institute of Mathematical Statistics},
  doi={10.1214/08-AOS637}
}

@inproceedings{anil2019sorting,
  title={Sorting Out {L}ipschitz Function Approximation},
  author={Anil, Cem and Lucas, James and Grosse, Roger},
  booktitle={Proceedings of the 36th International Conference on Machine Learning},
  series={Proceedings of Machine Learning Research},
  volume={97},
  pages={291--301},
  year={2019},
  eprint={1811.05381},
  archivePrefix={arXiv}
}

@inproceedings{fazlyab2019efficient,
  title={Efficient and Accurate Estimation of {L}ipschitz Constants for Deep Neural Networks},
  author={Fazlyab, Mahyar and Robey, Alexander and Hassani, Hamed and Morari, Manfred and Pappas, George J.},
  booktitle={Advances in Neural Information Processing Systems},
  volume={32},
  pages={11423--11434},
  year={2019},
  eprint={1906.04893},
  archivePrefix={arXiv}
}

@article{asiaee2026sharp,
  title={Sharp Bounds for Treatment Effect Generalization under Outcome Distribution Shift},
  author={Asiaee, Amir and Pal, Samhita and Beck, Cole and Huling, Jared D},
  journal={arXiv preprint arXiv:2602.09595},
  year={2026}
}

@article{asiaee2026omitted,
  title={Omitted-variable sensitivity analysis for generalizing randomized trials},
  author={Asiaee, Amir and Pal, Samhita and Huling, Jared D},
  journal={arXiv preprint arXiv:2603.27788},
  year={2026}
}

@inproceedings{kallus2018removing,
  title={Removing Hidden Confounding by Experimental Grounding},
  author={Kallus, Nathan and Puli, Aahlad Manas and Shalit, Uri},
  booktitle={Advances in Neural Information Processing Systems},
  volume={31},
  year={2018}
}

@misc{asiaee2026bcalmbiaslimitedbayesianborrowing,
      title={B-CALM: Bias-Limited Bayesian Borrowing for RCT-Anchored Treatment Effects under Covariate Mismatch},
      author={Amir Asiaee and Samhita Pal},
      year={2026},
      eprint={2607.04036},
      archivePrefix={arXiv},
      primaryClass={stat.ME},
      url={https://arxiv.org/abs/2607.04036}
}

\clearpage
\onecolumn
\makeatletter\global\let\@thanks\@empty\makeatother
\title{Improving RCT-Based Treatment Effect Estimation Under Covariate Mismatch via Calibrated Alignment\\(Supplementary Material)}
\maketitle

\appendix

\section{Proofs}\label{app:proofs}

\begin{proof}[Proof of Theorem~\ref{thm:calm-main}]
We decompose the error in three steps: the CATE calibration layer, the augmentation error, and the assembly.

\emph{Step 1 (CATE calibration layer).}
By the pseudo-outcome regression framework (\citet{asiaee2025improving}, Theorem~6), the CATE estimation error satisfies
\begin{equation}\label{eq:step1}
\Delta_2^2(\hat{\tau}_{\CALM}, \tau^r)
\leq \Delta_2^2(\cF, \tau^r)
+ C_1\Bigl(1 + \textstyle\sum_a \Dr{\hat{\mu}^{\mathrm{cal}}_a, \mu^r_a}\Bigr) \Rad_{n^r}^2(\cF) + C_2\frac{\log(1/\gamma)}{n^r},
\end{equation}
where $\mu^r_a$ is the RCT arm-specific mean from Section~\ref{sec:setup}, and $\hat{\mu}^{\mathrm{cal}}_a(\bm{x}^r) = \hat{\mu}^o_a(\hat{\phi}^r(\bm{x}^r)) + \hat{\delta}^t_a(\hat{\phi}^r(\bm{x}^r))$ is the calibrated prediction from Stage~2.
The first term is the approximation error of $\cF$; the second captures how augmentation quality changes the effective noise in pseudo-outcome regression.
Cross-fitting ensures that the nuisance estimates $\hat{\mu}^{\mathrm{cal}}_a$ are independent of the pseudo-outcome regression sample, allowing us to bound the stochastic term via Rademacher complexity and a concentration inequality.

\emph{Step 2 (Augmentation error decomposition).}
The per-arm augmentation error $\Dr{\hat{\mu}^{\mathrm{cal}}_a, \mu^r_a}$ is decomposed into estimation/alignment error and embedding sufficiency error:
\begin{align*}
\hat{\mu}^{\mathrm{cal}}_a(\bm{x}^r) - \mu^r_a(\bm{x}^r)
&= \underbrace{\hat{\mu}^o_a(\hat{\phi}^r) + \hat{\delta}^t_a(\hat{\phi}^r) - \tilde{\mu}^r_a(\hat{\phi}^r)}_{\text{(I): estimation and alignment}}
+ \underbrace{\tilde{\mu}^r_a(\hat{\phi}^r) - \mu^r_a(\bm{x}^r)}_{\text{(II): sufficiency}},
\end{align*}
where $\hat{\phi}^r = \hat{\phi}^r(\bm{x}^r)$ is shorthand.

\textit{Term (I):}
By Assumption~\ref{as:shift-embed}, $\tilde{\mu}^r_a = \tilde{\mu}^o_a + \delta_a$.
By cross-fitting, $\hat{\mu}^o_a$ is estimated on independent OS data and $\hat{\delta}^t_a$ on independent RCT data.
Standard empirical process arguments (e.g., \citet{bartlett2006local, wainwright2019high}) give the OS outcome-model and RCT discrepancy complexities.
Because the OS outcome model is trained on $\phi^o(\bX^o)$ but evaluated on $\phi^r(\bX^r)$, Assumption~\ref{as:alignment} and the Lipschitz constants in Assumption~\ref{as:alignment} add the transfer term $(L_\mu+L_\delta)^2r_\phi^2$.
Thus
\[
\E\!\left[(\hat{\mu}^o_a(\hat{\phi}^r(\bX^r)) + \hat{\delta}^t_a(\hat{\phi}^r(\bX^r)) - \tilde{\mu}^r_a(\hat{\phi}^r(\bX^r)))^2\mid S=r\right]
\leq
C\{\Rad_{n^o}^2(\cM^o_a)+\Rad_{n^r}^2(\cD_a)+(L_\mu+L_\delta)^2r_\phi^2\}.
\]

\textit{Term (II):}
Assumption~\ref{as:embed-sufficient} bounds the RCT-side embedding sufficiency error by $\epsilon_{\mathrm{suff}}^2$.

Combining Terms (I)--(II) via the Cauchy--Schwarz inequality:
\begin{equation}\label{eq:aug-error}
\Dr{\hat{\mu}^{\mathrm{cal}}_a, \mu^r_a} \leq C\bigl[\epsilon_{\mathrm{suff}}^2 + (L_\mu + L_\delta)^2 r_\phi^2 + \Rad_{n^o}^2(\cM^o_a) + \Rad_{n^r}^2(\cD_a)\bigr].
\end{equation}

\emph{Step 3 (Assembly).}
Substituting~\eqref{eq:aug-error} into~\eqref{eq:step1} yields the multiplicative term $\{\sum_a \Dr{\hat{\mu}^{\mathrm{cal}}_a,\mu^r_a}\}\Rad_{n^r}^2(\cF)$.
The bounded-envelope condition implies $\Rad_{n^r}^2(\cF)=O(1)$, so this product is bounded by a constant multiple of the augmentation-error terms in~\eqref{eq:aug-error}; the leading $\Rad_{n^r}^2(\cF)$ term remains as the CATE-regression complexity.
Summing over arms gives~\eqref{eq:calm-bound}.
The $\log(1/\gamma)/n^r$ term arises from a union bound and Bernstein-type concentration applied to the empirical process in the CATE calibration stage.
\end{proof}

\begin{proof}[Proof of Corollary~\ref{cor:calm-vs-mroscar}]
The \CALM{} bound in Theorem~\ref{thm:calm-main} differs from the \MROSCAR{} bound of \citet{pal2026mismatch} only in the terms used to make OS information usable under covariate mismatch, up to the comparable calibration and final CATE-regression complexities assumed in the corollary.
For \CALM{}, these terms are
\[
\epsilon_{\mathrm{suff}}^2+(L_\mu+L_\delta)^2r_\phi^2+\textstyle\sum_a\Rad_{n^o}^2(\cM^o_a).
\]
For \MROSCAR{}, the corresponding terms are the imputation error, the outcome-model complexity after imputation, and the complexity of the imputation class,
\[
L^2r_{\mathrm{im}}^2+\textstyle\sum_a\Rad_{n^o}^2(\cM^{o,\mathrm{im}}_a)+\Rad_{n^o}^2(\cG).
\]
Subtracting the common terms gives the displayed sufficient condition.
\end{proof}

\begin{proof}[Proof of Proposition~\ref{prop:safe}]
Condition on the training folds used to construct the augmentation $\hat{m}$, so that $\hat{m}(\bX^r)$ is fixed for the held-out RCT unit.
For any such augmentation,
\[
\E\!\left[\frac{A(Y-\hat{m}(\bX^r))}{\pi^r_A(\bX^r)}\Bigm|\bX^r,S=r\right]
=\{\mu^r_1(\bX^r)-\hat{m}(\bX^r)\}-\{\mu^r_{-1}(\bX^r)-\hat{m}(\bX^r)\}
=\tau^r(\bX^r),
\]
where the first equality uses RCT randomization and positivity.
Thus augmentation quality affects the variance of the pseudo-outcome, but not its conditional mean.
\end{proof}

\begin{proof}[Proof of Proposition~\ref{prop:linear-equiv}]
Let
\[
\bm{M}=
\begin{pmatrix}
\bm{0}_{p_z \times p_u} & \bm{I}_{p_z} \\
\bm{0}_{p_v \times p_u} & \bm{\Lambda}
\end{pmatrix}.
\]
Under the Gaussian model and $\bU\ind\bV\mid\bZ$, the $L_2$-optimal linear prediction of the OS covariate vector from the RCT covariates is
\[
\E[(\bZ,\bV)^\top\mid \bU,\bZ]=\bm{M}(\bU,\bZ)^\top,
\]
because $\E[\bV\mid\bU,\bZ]=\E[\bV\mid\bZ]=\bm{\Lambda}\bZ$.
For a fixed OS projection $\bW^o$, the best RCT-side linear embedding is therefore the conditional mean of the OS embedding given $\bX^r$:
\[
\E[\bW^o(\bZ,\bV)^\top\mid \bX^r]=\bW^o\bm{M}(\bX^r)^\top.
\]
Thus $\bW^r_{\mathrm{opt}}=\bW^o\bm{M}$, which is exactly the embedding obtained by first imputing $\widehat{\bV}=\bm{\Lambda}\bZ$ and then applying $\bW^o$.
\end{proof}

\begin{proof}[Proof of Corollary~\ref{cor:linear-calm}]
Apply Theorem~\ref{thm:calm-main} to the linear embedding in Proposition~\ref{prop:linear-equiv}.
The residual alignment error is the conditional variation in the missing block after linear imputation, so it is controlled, up to the linear projection constants absorbed by $\lesssim$, by $\mathrm{tr}(\bm{\Sigma}_{\bV\mid\bZ})$.
Standard Rademacher bounds for $s$-sparse linear classes give squared complexities of order $s\log d/n^o$ for the OS outcome model, $s_\delta\log d/n^r$ for the calibration discrepancy, and $s_\tau\log p_r/n^r$ for the final CATE regression.
Substitution gives the displayed rate.
\end{proof}

\section{Alignment Objective Details}\label{app:alignment-objectives}

Throughout this section we write $\bm{w}^o_j = \hat{\phi}^o(\bX^o_j)$ for the frozen OS embeddings and $\bm{w}^r_i = \phi^r(\bX^r_i)$ for the learnable RCT embeddings, both in $\R^d$.

\subsection{Distribution-Matching Alignment (MMD)}

Instantiating~\eqref{eq:mmd-align} with the Gaussian RBF kernel
$k(\bm{w},\bm{w}') = \exp\!\bigl(-\|\bm{w}-\bm{w}'\|^2/(2\sigma^2)\bigr)$
and expanding the squared RKHS norm yields the unbiased U-statistic estimate
\begin{align}\label{eq:mmd-expanded}
\cL_{\mathrm{MMD}} = \frac{1}{n^o(n^o\!-\!1)}\sum_{\substack{j,j' \in \text{OS}\\j\neq j'}} \exp\!\Bigl(-\frac{\|\bm{w}^o_j - \bm{w}^o_{j'}\|^2}{2\sigma^2}\Bigr) 
&- \frac{2}{n^o n^r}\sum_{j \in \text{OS}}\sum_{i \in \text{RCT}} \exp\!\Bigl(-\frac{\|\bm{w}^o_j - \bm{w}^r_i\|^2}{2\sigma^2}\Bigr) \\ \nonumber
&+ \frac{1}{n^r(n^r\!-\!1)}\sum_{\substack{i,i' \in \text{RCT}\\i\neq i'}} \exp\!\Bigl(-\frac{\|\bm{w}^r_i - \bm{w}^r_{i'}\|^2}{2\sigma^2}\Bigr).
\end{align}
The bandwidth $\sigma$ is set via the median heuristic: $\sigma = \sqrt{\mathrm{median}\{\|\bm{w}_i - \bm{w}_j\|^2 : i \neq j\}}$, computed over the pooled embeddings from both sources \citep{gretton2012kernel}.
Since $\hat{\phi}^o$ is frozen, only the cross-term and the RCT--RCT term contribute gradients with respect to $\phi^r$.

\subsection{\texorpdfstring{$\bZ$}{Z}-Conditioned Contrastive Alignment}

The contrastive loss~\eqref{eq:contrastive-align} pairs each RCT unit $i$ with its OS neighbours $N_i = \{j \in \text{OS} : \|\bZ_j - \bZ_i\| \leq \varepsilon\}$ in shared-covariate space and penalizes the squared embedding distance:
\begin{equation}\label{eq:contrastive-expanded}
\cL_{\mathrm{contr}} = \frac{1}{n^r}\sum_{i \in \text{RCT}} \frac{1}{|N_i|} \sum_{j \in N_i} \|\bm{w}^o_j - \bm{w}^r_i\|^2.
\end{equation}
This objective is already in closed form; the neighbourhood radius $\varepsilon$ trades off bias (large $\varepsilon$ matches dissimilar units) against variance (small $\varepsilon$ yields few or no neighbours).

\subsection{Adversarial Domain Alignment}

Following \citet{ganin2016domain}, we introduce a domain discriminator $g_\omega : \R^d \to [0,1]$ trained to distinguish OS from RCT embeddings, while the RCT encoder $\phi^r$ is trained to fool it.
Let $s_j = 1$ for OS units and $s_i = 0$ for RCT units.
The adversarial alignment objective is
\begin{equation}\label{eq:adversarial-align}
\cL_{\mathrm{adv}} = -\frac{1}{n^o}\sum_{j \in \text{OS}} \log g_\omega(\bm{w}^o_j) 
- \frac{1}{n^r}\sum_{i \in \text{RCT}} \log\bigl(1 - g_\omega(\bm{w}^r_i)\bigr),
\end{equation}
which is maximized over $\omega$ and minimized over $\phi^r$ via a gradient reversal layer \citep{ganin2016domain}.
At the minimax equilibrium, $g_\omega \equiv 1/2$, implying that the embedding distributions are indistinguishable.
{We omit adversarial alignment from the experiments} because it introduces a second network and alternating optimization, making training less stable; we include it here for completeness.

\subsection{Implementation Notes}

In the experiments, \CALM-NN uses a \emph{conditional-mean alignment} variant, which fits a ridge regression $\bZ \mapsto \E[\hat{\phi}^o(\bZ, \bV) \mid \bZ]$ on the OS and uses the predicted embeddings as alignment targets for $\phi^r$:
\begin{equation}\label{eq:cond-mean-align}
\cL_{\mathrm{cond}} = \frac{1}{n^r}\sum_{i \in \text{RCT}} \bigl\|\bm{w}^r_i - \hat{h}(\bZ_i)\bigr\|^2, \quad \hat{h} = \argmin_{h \in \cH_{\mathrm{ridge}}} \frac{1}{n^o}\sum_{j \in \text{OS}} \|h(\bZ_j) - \bm{w}^o_j\|^2 + \alpha\|h\|^2.
\end{equation}
This variant avoids kernel bandwidth selection and is computationally cheaper but assumes a linear relationship between $\bZ$ and the OS embeddings.
The alignment weight $\lambda$ is annealed during Stage~2 training: held constant for the first $60\%$ of epochs, then linearly decayed to $0.2\lambda$ over the remaining $40\%$.

\section{Additional Experimental Results}\label{app:experiments}

\subsection{Baseline Regime}

This subsection {collects} additional sweeps from the baseline (linear-CATE) regime described in Section~\ref{sec:experiments}.
The DGP uses $p_z = 30$, $p_u = 10$, $p_v = 20$, linear outcomes, and default parameters $n^r = 500$, $n^o = 10{,}000$, $\sigma_V^2 = 1.0$, $d_{\mathrm{true}} = 5$, shift magnitude $0.5$.
Each figure varies one factor while {keeping the remaining settings} at defaults.
Mean RMSE is computed over 20 replicates.
Across all settings, the four calibration-based methods (\RACER, \SROSCAR, \MROSCAR, \CALM-Lin) remain effectively indistinguishable (pairwise RMSE differences below $10^{-3}$).

{\textbf{Intrinsic dimension and shared proportion (Figures~\ref{fig:rmse-dim},~\ref{fig:rmse-shared}).}\ 
No single method dominates as $d_{\mathrm{true}}$ varies; the winner alternates among \RACER{} ($d_{\mathrm{true}} \in \{2, 15\}$), \SROSCAR{} ($10$ and $20$), \MROSCAR{} ($3$), and \CALM-Lin{} ($5$), with small gaps throughout.
Similarly, as the shared covariate proportion increases, all methods improve and the identity of the best calibration-based method shifts: \CALM-Lin{} at low and moderate overlap ($0.3$ and $0.5$), \MROSCAR{} at intermediate overlap ($0.7$), and \RACER{} at high overlap ($0.9$).}

{\textbf{RCT sample size (Figure~\ref{fig:rmse-nr}).}\ 
At $n^r = 100$ (baseline linear-CATE regime), HTCE-T yields the lowest RMSE ($2.84$), while the OS-borrowing calibration variants (\SROSCAR{}, \MROSCAR{}, \CALM-Lin) {show instability}.
For $n^r \geq 250$ the calibration-based methods converge; at $n^r = 2{,}000$, \RACER{}, \SROSCAR{}, \MROSCAR{}, and \CALM-Lin{} all achieve RMSE $\approx 0.53$.}

{\textbf{Outcome-model nonlinearity and negative transfer (Figures~\ref{fig:nonlinear},~\ref{fig:neg-transfer}).}\ 
Under outcome nonlinearity, \RACER{} is best for linear outcomes (RMSE $1.15$), \CALM-Lin{} for quadratic ($1.45$), and \MROSCAR{} for sinusoidal ($1.06$); \CALM-NN does not improve over the strongest linear baselines here.
Under increasing outcome shift, calibration-based methods remain stable (e.g., RMSE $\approx 1.11$ for \CALM-Lin{} at shift $5.0$), whereas Naive inflates to $3.84$ and HTCE-T to $3.32$.
\CALM-NN is less robust at extreme shift (RMSE $2.25$ at shift $5.0$), consistent with the limitation that the alignment objective can distort the learned representation when the domain gap is large.}

\begin{figure}[t]
\centering
\includegraphics[width=\columnwidth]{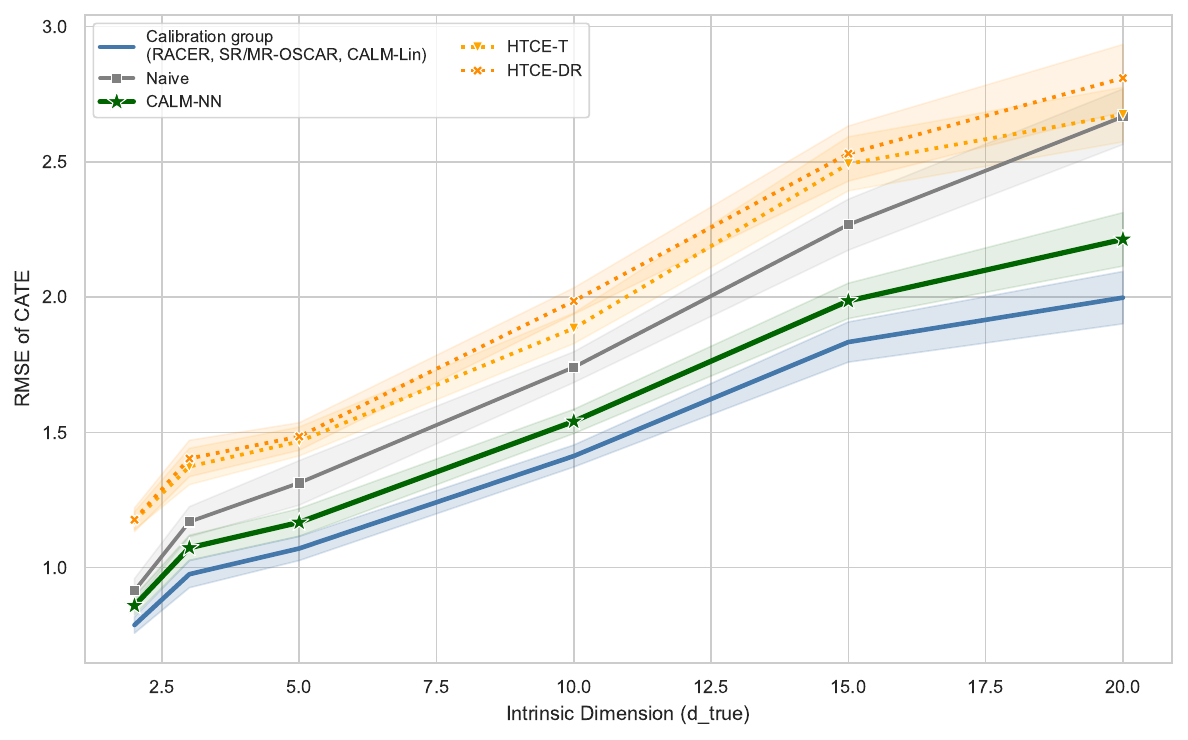}
\caption{RMSE of CATE estimation as a function of intrinsic dimension $d_{\mathrm{true}}$, with $\sigma_V^2 = 1.0$, $n^r = 500$, and linear outcome model. Mean over 20 replicates. Blue band: calibration-group envelope (see Figure~\ref{fig:main-results} caption). No single method dominates uniformly across intrinsic dimensions.}
\label{fig:rmse-dim}
\end{figure}

\begin{figure}[t]
\centering
\includegraphics[width=\columnwidth]{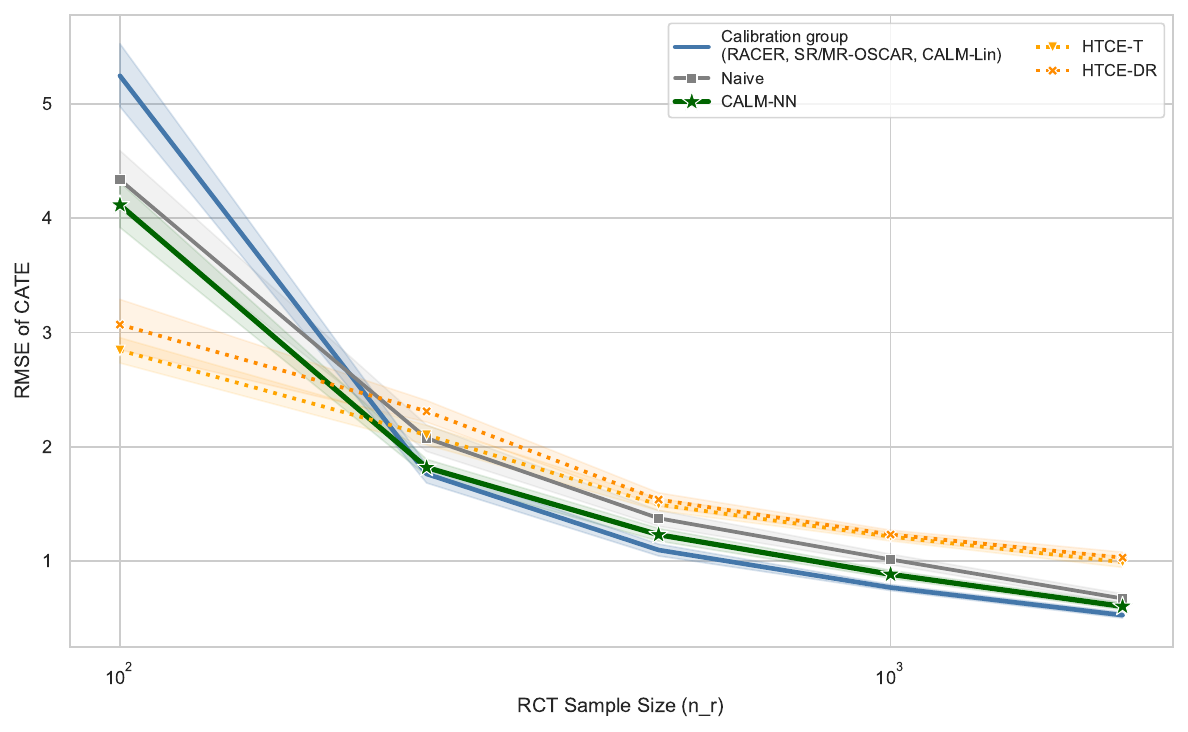}
\caption{RMSE of CATE estimation as a function of RCT sample size $n^r$, with $n^o = 10{,}000$ fixed, $\sigma_V^2 = 1.0$, $d_{\mathrm{true}} = 5$, and linear outcome model. Mean over 20 replicates. Blue band: calibration-group envelope (see Figure~\ref{fig:main-results} caption). HTCE methods perform best at the smallest $n^r$, while calibration-based methods converge as $n^r$ grows.}
\label{fig:rmse-nr}
\end{figure}

\begin{figure}[t]
\centering
\includegraphics[width=\columnwidth]{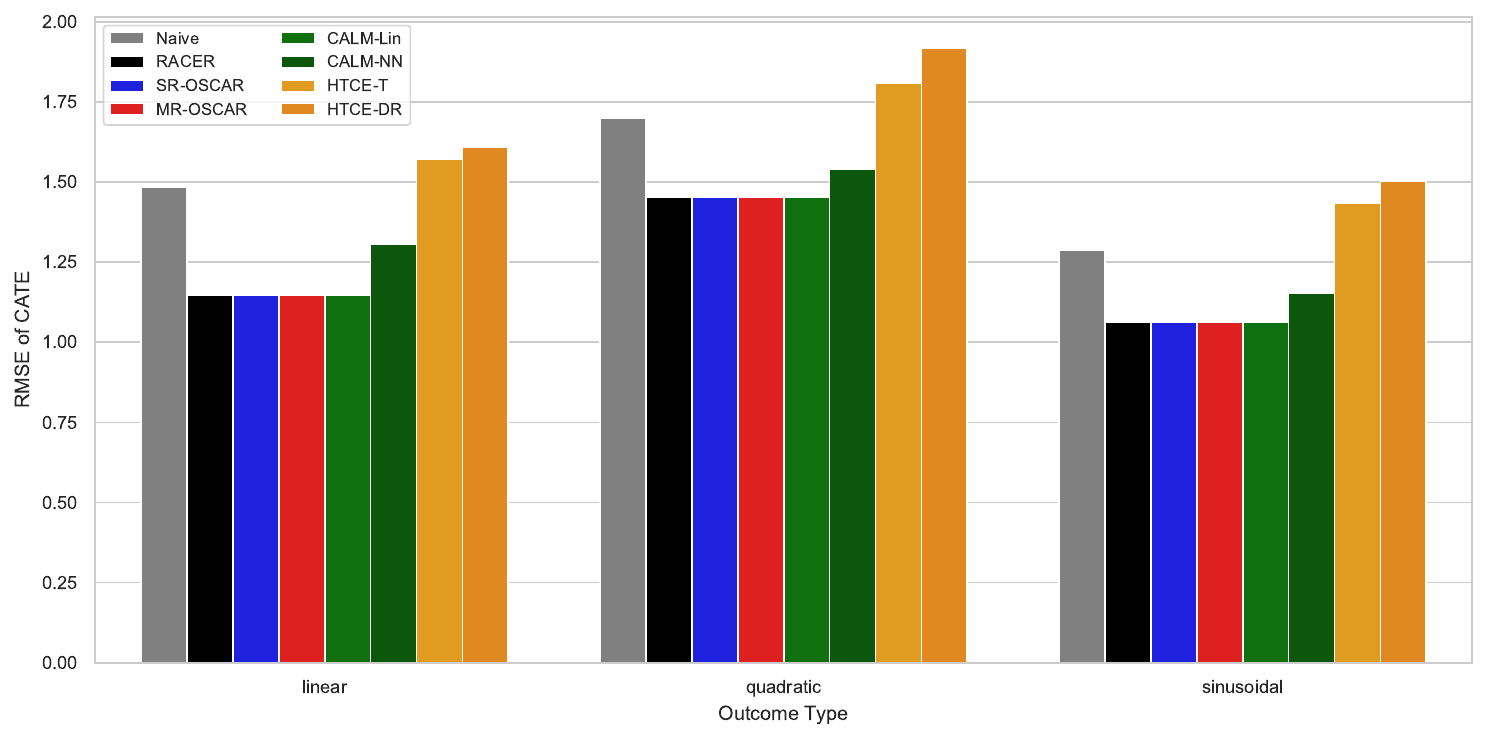}
\caption{RMSE across outcome model types (linear, quadratic, sinusoidal), with $\sigma_V^2 = 1.0$, $n^r = 500$, $d_{\mathrm{true}} = 5$. Bars show mean RMSE over 20 replicates. The best method depends on the outcome type.}
\label{fig:nonlinear}
\end{figure}

\begin{figure}[t]
\centering
\includegraphics[width=\columnwidth]{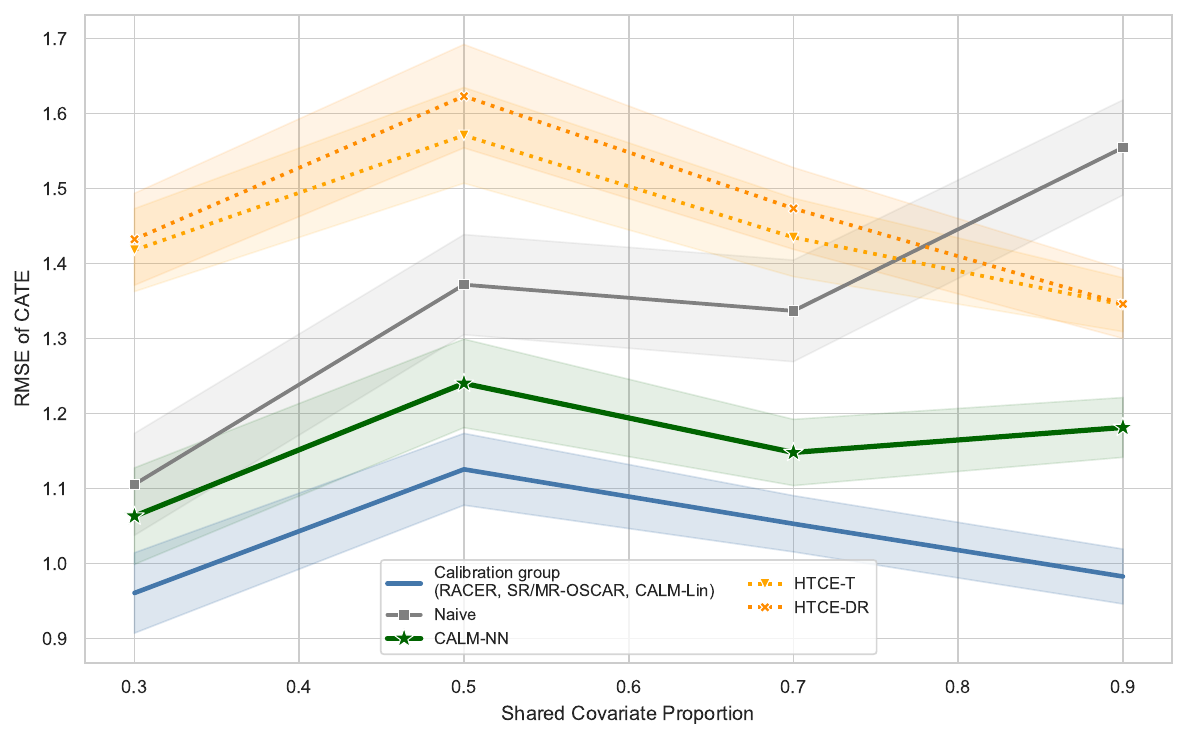}
\caption{RMSE as a function of shared covariate proportion $p_z / (p_z + p_v)$, with $\sigma_V^2 = 1.0$, $n^r = 500$, $d_{\mathrm{true}} = 5$, and linear outcome model. Mean over 20 replicates. Blue band: calibration-group envelope (see Figure~\ref{fig:main-results} caption). Performance improves as shared proportion increases and the mismatch problem weakens.}
\label{fig:rmse-shared}
\end{figure}

\begin{figure}[t]
\centering
\includegraphics[width=\columnwidth]{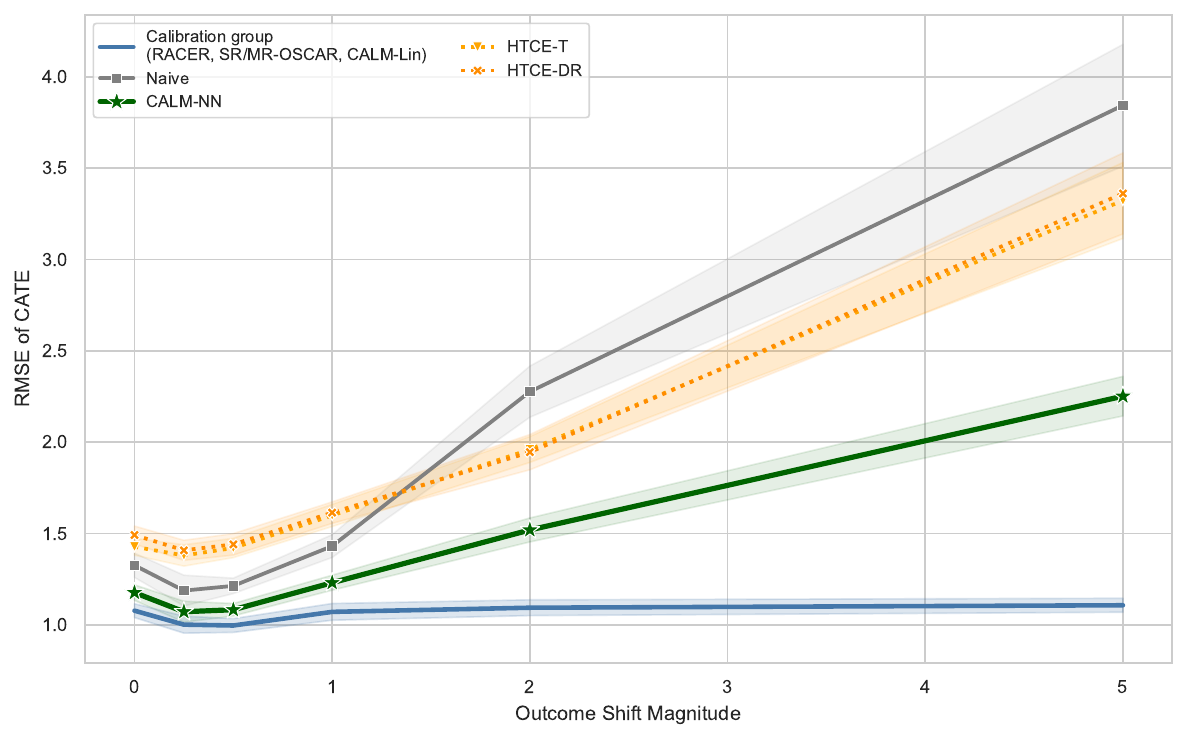}
\caption{RMSE under increasing outcome shift between OS and RCT, with $\sigma_V^2 = 1.0$, $n^r = 500$, and linear outcome model. Mean over 20 replicates. Blue band: calibration-group envelope (see Figure~\ref{fig:main-results} caption). Calibration-based methods remain stable under shift, while Naive and HTCE methods inflate under large shift.}
\label{fig:neg-transfer}
\end{figure}

\subsection{Nonlinear-CATE Regime}

This subsection {collects} additional sweeps from the nonlinear-CATE regime described in Section~\ref{sec:results-nonlinear}.
The DGP uses a shared-latent structure: $\bU$ and $\bV$ share a 5-dimensional latent factor ($\sigma_V^2 = \sigma_U^2 = 0.1$), outcomes depend only on $\bV$ ($w_Z = w_U = 0$, $w_V = 2$), and the CATE takes a sinusoidal form with frequency $\omega$.
Mean RMSE is computed over 20 replicates.
\CALM-NN {has} the lowest RMSE in all settings reported here.

{\textbf{Robustness to signal routing and CATE functional form (Figures~\ref{fig:signal-weight-z},~\ref{fig:cate-form}).}\ 
To test whether the advantage depends on the outcome signal being routed entirely through unobserved covariates $\bV$, we vary the shared-covariate signal weight $w_Z \in \{0, 0.5, 1.0, 1.5, 2.0\}$ while holding $w_V = 2$ and $\omega = 1.5$.
\CALM-NN wins all 5 settings with margins of $0.45$--$0.59$ over calibration-based methods, and the advantage does not diminish as $w_Z$ increases (RMSE $0.62$ vs.\ $1.12$ at $w_Z = 2$).
{This pattern suggests that} the gain stems from the ability to model nonlinear CATEs, not from exploiting $\bV$-specific information inaccessible to imputation-based approaches.
The advantage also generalizes across CATE functional forms (absolute-value, quadratic; Figure~\ref{fig:cate-form}).}

{\textbf{Latent coupling strength (Figure~\ref{fig:latent-scale}).}\ 
We vary $\alpha_U \in \{0.5, 1.0, 2.0, 3.0, 4.0\}$, controlling how much $\bU$ encodes information about $\bV$ through the shared latent.
The largest margin appears at $\alpha_U = 0.5$ (RMSE $0.57$ vs.\ $1.58$); as coupling increases, calibration-based methods improve (from $1.58$ to $1.27$) while \CALM-NN slightly increases ($0.57$ to $0.87$), but \CALM-NN remains the best method throughout.
This pattern is consistent with the theory: when the latent channel is narrow, imputation is difficult and the nonlinear encoder has a greater relative advantage.}

\begin{figure}[t]
\centering
\includegraphics[width=\columnwidth]{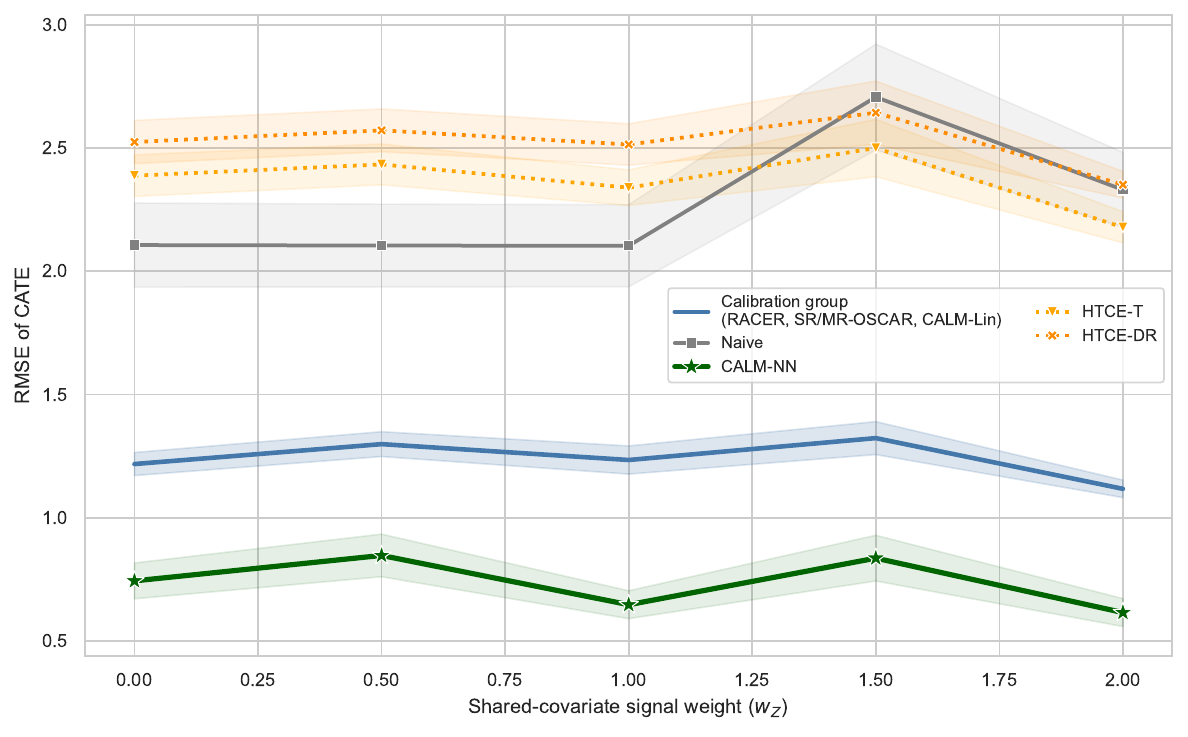}
\caption{RMSE as a function of shared-covariate signal weight $w_Z$, with $w_V = 2$, $\omega = 1.5$, and the nonlinear-CATE DGP. Mean over 20 replicates. Blue band: calibration-group envelope (see Figure~\ref{fig:main-results} caption). \CALM-NN's advantage persists even when $\bZ$ contributes substantially to the outcome signal.}
\label{fig:signal-weight-z}
\end{figure}

\begin{figure}[t]
\centering
\includegraphics[width=\columnwidth]{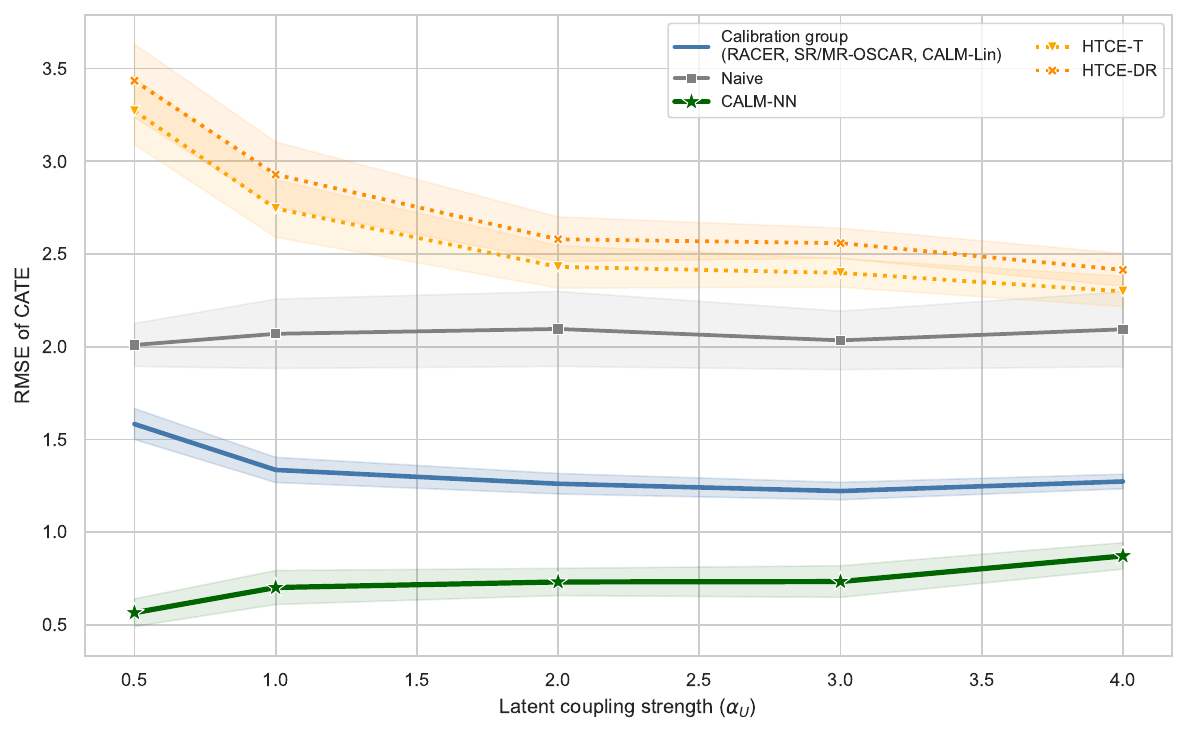}
\caption{RMSE as a function of latent coupling strength $\alpha_U$, which controls how much $\bU$ encodes information about $\bV$ through shared latent factors. Nonlinear-CATE DGP with $\omega = 1.5$, $w_Z = 0$. Mean over 20 replicates. Blue band: calibration-group envelope (see Figure~\ref{fig:main-results} caption).}
\label{fig:latent-scale}
\end{figure}

\begin{figure}[t]
\centering
\includegraphics[width=\columnwidth]{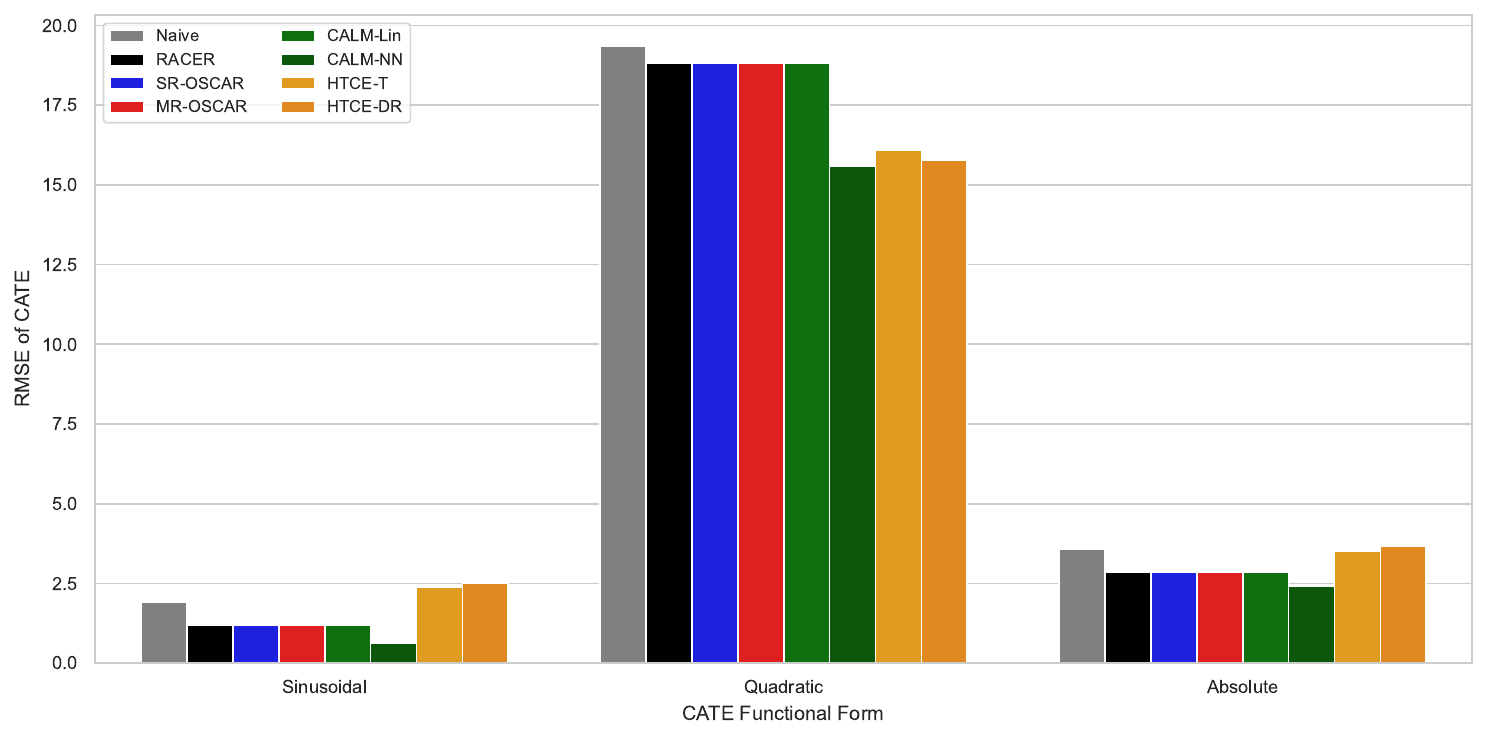}
\caption{RMSE across three nonlinear CATE functional forms (sinusoidal, quadratic, absolute value) in the nonlinear-CATE DGP ($\omega = 1.5$, $w_Z = 0$). Bars show mean RMSE over 20 replicates. \CALM-NN (dark green) achieves the lowest RMSE across all three forms.}
\label{fig:cate-form}
\end{figure}

Under absolute-value CATEs, \CALM-NN achieves RMSE $2.41$ versus $2.86$ for the best calibration-based method (margin $0.45$); under quadratic CATEs, $15.60$ versus $18.83$ (margin $3.23$).
The quadratic form {has high RMSE} for all methods due to the unbounded square function, but \CALM-NN still attains the lowest RMSE (sinusoidal: $0.63$ vs.\ $1.18$, margin $0.55$).

\subsection{Semi-synthetic Benchmark}

We construct a semi-synthetic benchmark from the IHDP dataset \citep{hill2011bayesian}, splitting covariates into $\bZ$, $\bU$, and $\bV$ and bootstrap-sampling units to form OS and RCT datasets.
Outcomes follow a structural model with a treatment effect $\tau(\bX^r)$ and an RCT-only shift $\delta^r_a(\bZ)$ to induce domain mismatch.
All eight methods from Section~\ref{sec:experiments} are evaluated over 50 replicates.
Results {show} the calibration-based equivalence observed in the linear-CATE simulations: all four calibration methods achieve RMSE $\approx 0.205$, with \CALM-NN at $0.23$ (Figure~\ref{fig:ihdp}).

\begin{figure}[t]
\centering
\includegraphics[width=\columnwidth]{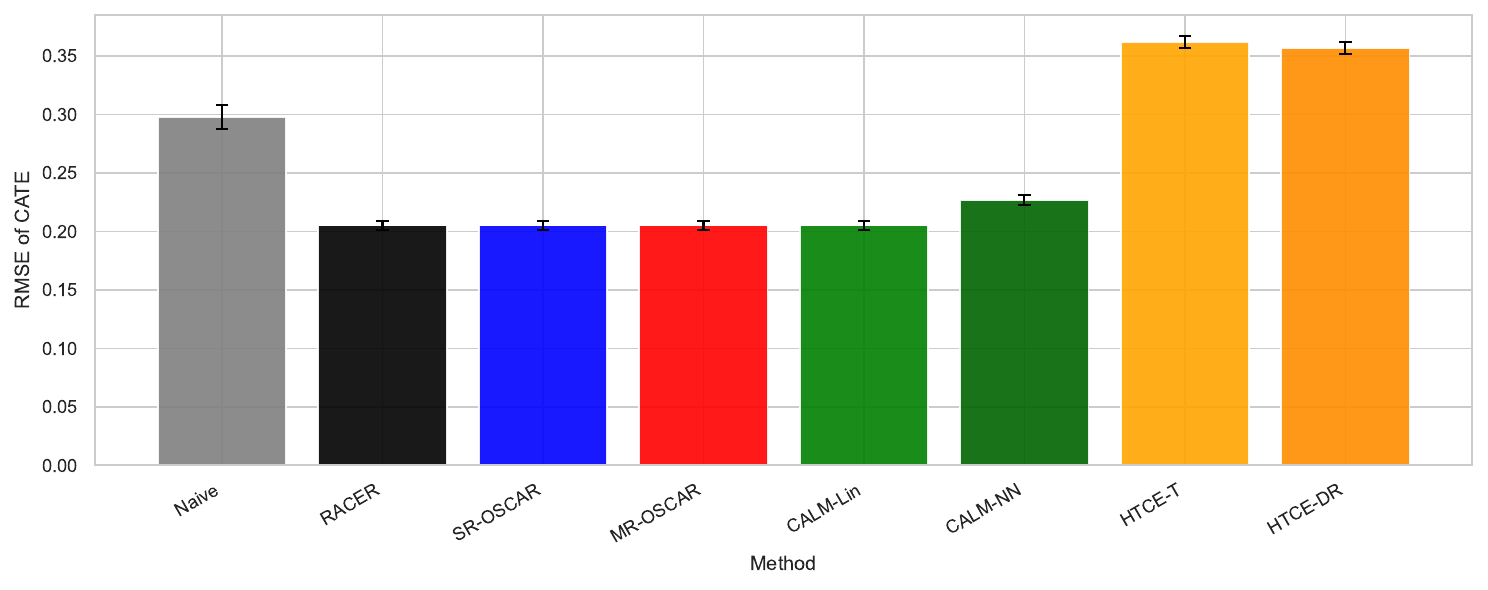}
\caption{Semi-synthetic benchmark using IHDP covariates: mean RMSE of CATE estimation over 50 replicates (error bars: $\pm 1$ SE).}
\label{fig:ihdp}
\end{figure}

\section{Additional Theory for Observable Diagnostics}\label{app:observable-diagnostics}

\subsection{Observable diagnostics for the alignment and sufficiency terms}

The leading nonstandard terms in Theorem~\ref{thm:calm-main} are
$r_\phi^2$ and $\epsilon_{\mathrm{suff}}^2$.  The first is not directly
observable because $\bV$ is unmeasured in the trial; the second is a
population approximation error.  Under additional conditions, the
following held-out quantities upper bound or empirically track the
corresponding terms in the bound.

\begin{lemma}[Alignment proxy for $r_\phi^2$]\label{lem:cr-rphi-proxy}
Assume Assumption~\ref{as:rct}.  Suppose that, under the target RCT
population, $\bU \ind \bV^*\mid\bZ$; the support of $\bZ$ is compact; the
encoders $\phi^o$ and $\phi^r$ are Lipschitz on the relevant compact
covariate supports; the
OS and RCT conditional laws of $\bV\mid\bZ$ agree for the counterfactual
value $\bV^*$; and the OS and RCT $\bZ$ marginals coincide, or else the
OS conditional expectations below are reweighted to the RCT $\bZ$
marginal.  Let
\[
m_o(\bm{z})=\E\{\phi^o(\bZ,\bV)\mid \bZ=\bm{z},S=o\},
\qquad
m_r(\bm{z})=\E\{\phi^r(\bZ,\bU)\mid \bZ=\bm{z},S=r\}.
\]
Define
\begin{align*}
B_{\mathrm{contr}}
&=
\E\!\left[\|\phi^r(\bX^r)-m_o(\bZ)\|^2\mid S=r\right],\\
B_{\mathrm{OS}}
&=
\E\!\left[\|\phi^o(\bX^o)-m_o(\bZ)\|^2\mid S=o\right],\\
B_{\mathrm{RCT}}
&=
\E\!\left[\|\phi^r(\bX^r)-m_r(\bZ)\|^2\mid S=r\right].
\end{align*}
Then
\[
r_\phi^2
\leq
3\{B_{\mathrm{contr}}+B_{\mathrm{OS}}+B_{\mathrm{RCT}}\}.
\]
If the $\bZ$ marginals differ across sources, the same display holds
with $B_{\mathrm{OS}}$ evaluated under the RCT $\bZ$ marginal, equivalently
by weighting OS units by $dP(\bZ\mid S=r)/dP(\bZ\mid S=o)$ when this ratio
exists and is bounded.
\end{lemma}

\begin{proof}
For an RCT unit, couple the unobserved $\bV^*$ through its conditional law
given $\bZ$.  By conditional transportability of $\bV^*\mid\bZ$ and the
definition of $m_o$,
\[
\E\!\left[\|\phi^o(\bZ,\bV^*)-m_o(\bZ)\|^2\mid S=r\right]
=B_{\mathrm{OS}},
\]
up to the density-ratio weighting noted in the statement.  Now write
\[
\phi^o(\bZ,\bV^*)-\phi^r(\bX^r)
=
\{\phi^o(\bZ,\bV^*)-m_o(\bZ)\}
+\{m_o(\bZ)-\phi^r(\bX^r)\}.
\]
The two-term inequality $\|u+v\|^2\leq 2\|u\|^2+2\|v\|^2$ gives
\[
r_\phi^2\leq 2(B_{\mathrm{OS}}+B_{\mathrm{contr}})
\leq 3(B_{\mathrm{OS}}+B_{\mathrm{contr}}+B_{\mathrm{RCT}}),
\]
where the last step adds the nonnegative RCT residual term.  This third
term is not needed for the minimal upper bound, but it gives a symmetric
diagnostic for RCT-private variation in the learned embedding.
\end{proof}

\textbf{Empirical diagnostic.}
We estimate $B_{\mathrm{contr}}$ as the held-out distance from each RCT
embedding to the OS conditional mean embedding $m_o(\bZ)$,
$B_{\mathrm{OS}}$ as the OS residual around $m_o(\bZ)$, and
$B_{\mathrm{RCT}}$ as the RCT residual around $m_r(\bZ)$.  The unscaled
sum $B_{\mathrm{contr}}+B_{\mathrm{OS}}+B_{\mathrm{RCT}}$ had Pearson
correlation $0.9952$ with the oracle $r_\phi^2$ over 102 \CALM-Lin{} and
\CALM-NN{} settings.

\begin{lemma}[Cross-source residual proxy]\label{lem:cr-cross-source-proxy}
Fix an arm $a$.  Let
\[
R_{\mathrm{raw},a}^2
=
\E\!\left[
\{Y-\hat{\mu}^o_a(\hat{\phi}^r(\bX^r))\}^2
\mid S=r,A=a
\right]
\]
and
\[
R_{\mathrm{cal},a}^2
=
\E\!\left[
\{Y-\hat{\mu}^o_a(\hat{\phi}^r(\bX^r))
-\hat{\delta}^t_a(\hat{\phi}^r(\bX^r))\}^2
\mid S=r,A=a
\right],
\]
both estimated on held-out RCT folds.  Suppose
Assumptions~\ref{as:rct}--\ref{as:alignment} hold, the nuisance estimates
are $L_2(P_r)$ consistent for their population limits, and
\[
\sigma_{\mathrm{noise},a}^2
=
\E[\Var\{Y(a)\mid \bX^r,S=r\}\mid S=r,A=a]
<\infty.
\]
Then
\begin{align*}
R_{\mathrm{raw},a}^2-\sigma_{\mathrm{noise},a}^2
&\leq
2\E[\delta_a^2(\phi^r(\bX^r))\mid S=r,A=a]
+4\epsilon_{\mathrm{suff}}^2
+4(L_\mu+L_\delta)^2r_{\phi,a}^2
+o_p(1),\\
R_{\mathrm{cal},a}^2-\sigma_{\mathrm{noise},a}^2
&\leq
2\epsilon_{\mathrm{suff}}^2
+2(L_\mu+L_\delta)^2r_{\phi,a}^2
+o_p(1),
\end{align*}
where
\[
r_{\phi,a}^2
=
\E[\|\phi^o(\bX^{o,*})-\phi^r(\bX^r)\|^2\mid S=r,A=a].
\]
\end{lemma}

\begin{proof}
Work at the population nuisance limits; the held-out consistency
assumptions add the $o_p(1)$ terms.  For the raw residual,
Assumption~\ref{as:shift-embed} gives the exact three-term decomposition
\[
Y(a)-\tilde{\mu}^o_a(\phi^r(\bX^r))
=
\underbrace{Y(a)-\mu^r_a(\bX^r)}_{\eta_a}
+\underbrace{\mu^r_a(\bX^r)-\tilde{\mu}^r_a(\phi^r(\bX^r))}_{e^r_{\mathrm{suff},a}}
+\underbrace{\tilde{\mu}^r_a(\phi^r(\bX^r))
-\tilde{\mu}^o_a(\phi^r(\bX^r))}_{\delta_a(\phi^r(\bX^r))}.
\]
The noise term $\eta_a$ has conditional mean zero given
$\bX^r,S=r,A=a$, and therefore contributes
$\sigma_{\mathrm{noise},a}^2$.  Thus
\[
R_{\mathrm{raw},a}^2-\sigma_{\mathrm{noise},a}^2
\leq
2\E[(e^r_{\mathrm{suff},a})^2\mid S=r,A=a]
+2\E[\delta_a^2(\phi^r(\bX^r))\mid S=r,A=a]
+o_p(1).
\]
Assumption~\ref{as:embed-sufficient} controls the RCT-side sufficiency
residual.  The displayed bounds keep the alignment contribution from
Theorem~\ref{thm:calm-main} as a nonnegative diagnostic term, which
allows the residual check to be read on the same scale as the main
bound.  Substitution gives the raw-residual bound.  For the calibrated
residual, the discrepancy term is subtracted:
\[
Y(a)-\{\tilde{\mu}^o_a(\phi^r(\bX^r))+\delta_a(\phi^r(\bX^r))\}
=
\eta_a+e^r_{\mathrm{suff},a}.
\]
The same argument for $e^r_{\mathrm{suff},a}$ gives the second display.
No equality up to $o_p(1)$ is claimed, since these residual components
are not orthogonal in general.
\end{proof}

\textbf{Empirical diagnostic.}
We report the raw and calibrated held-out RCT residuals $R_{\mathrm{raw},a}^2$
and $R_{\mathrm{cal},a}^2$ to check whether transferred OS predictions
remain accurate after calibration.  Their alignment component is read
alongside the three-term alignment proxy above, whose unscaled sum had
Pearson correlation $0.9952$ with the oracle $r_\phi^2$ over 102 settings.

\begin{lemma}[OS residual proxy for $\epsilon_{\mathrm{suff}}^2$]\label{lem:cr-suff-proxy}
Fix an arm $a$.  Let
\[
R_{\mathrm{suff,OS},a}^2
=
\E\!\left[
\{Y-\hat{\mu}^o_a(\hat{\phi}^o(\bX^o))\}^2
\mid S=o,A=a
\right],
\]
estimated on a held-out OS fold independent of the fold used to train
$(\hat{\phi}^o,\hat{\mu}^o_a)$.  Suppose
Assumption~\ref{as:embed-sufficient} holds for the OS source,
$\hat{\mu}^o_a\circ\hat{\phi}^o$ is $L_2(P_o)$ consistent for
$\tilde{\mu}^o_a\circ\phi^o$, and
\[
\sigma_{\mathrm{noise},o,a}^2
=
\E[\Var\{Y(a)\mid \bX^o,S=o\}\mid S=o,A=a]
<\infty.
\]
Then
\[
R_{\mathrm{suff,OS},a}^2-\sigma_{\mathrm{noise},o,a}^2
=
\E\!\left[
\{\mu^o_a(\bX^o)-\tilde{\mu}^o_a(\phi^o(\bX^o))\}^2
\mid S=o,A=a
\right]+o_p(1)
\leq
\epsilon_{\mathrm{suff}}^2+o_p(1).
\]
\end{lemma}

\begin{proof}
On the held-out fold,
\[
Y(a)-\tilde{\mu}^o_a(\phi^o(\bX^o))
=
\{Y(a)-\mu^o_a(\bX^o)\}
+\{\mu^o_a(\bX^o)-\tilde{\mu}^o_a(\phi^o(\bX^o))\}.
\]
The first term has conditional mean zero given $\bX^o,S=o,A=a$, so its
cross-product with the second term has expectation zero.  The first term
therefore contributes $\sigma_{\mathrm{noise},o,a}^2$, and the second is
the embedding sufficiency residual.  Held-out $L_2(P_o)$ consistency
replaces the population predictor by
$\hat{\mu}^o_a\circ\hat{\phi}^o$ at an $o_p(1)$ cost, and
Assumption~\ref{as:embed-sufficient} bounds the residual by
$\epsilon_{\mathrm{suff}}^2$.
\end{proof}

\textbf{Empirical diagnostic.}
This held-out OS outcome residual is the sufficiency-side diagnostic
read alongside the calibrated RCT residual.  It does not itself target
$r_\phi^2$, which is checked by the alignment proxy above.

\subsection{Calibration of the constant $C$}

In the sparse linear setting of Corollary~\ref{cor:linear-calm}, the
constant in Theorem~\ref{thm:calm-main} can be calibrated using standard
bounded-design and sub-Gaussian-noise quantities.  Suppose that, for each
arm $a$, $Y(a)-\mu^r_a(\bX^r)$ is conditionally
$\sigma_Y^2$-sub-Gaussian, the RCT overlap constant in
Assumption~\ref{as:rct} satisfies
$\rho\leq\pi^r_a(\bX^r)\leq 1-\rho$, and the normalized linear design has
empirical Gram eigenvalues bounded away from zero and infinity.  A
Bernstein inequality for the inverse-propensity-weighted pseudo-outcome
noise yields, with probability at least $1-\gamma$,
\[
\sup_{f\in\cF}
\left|
\frac{1}{n^r}\sum_{i:S_i=r}
\{f(\bX^r_i)-\tau^r(\bX^r_i)\}\xi_i
\right|
\leq
c_0\frac{\sigma_Y}{\rho}
\sqrt{\frac{\log(p_o+p_r)+\log(1/\gamma)}{n^r}},
\]
where $\xi_i$ is the centered pseudo-outcome noise and $c_0$ is a
universal constant.  Squaring and absorbing the usual empirical-process
constants gives
\[
C
\leq
c\,\frac{\sigma_Y^2}{\rho^2}\log(p_o+p_r).
\]

For the representative simulation configuration
\[
\sigma_Y\approx1,\qquad
\rho=0.5,\qquad
p_o=50,\qquad
p_r=40,\qquad
n^r=500,
\]
we have $\log(p_o+p_r)=\log(90)\approx4.50$.  Taking the conservative
constant $c=16$ gives
\[
C\leq 16\cdot\frac{1}{0.5^2}\cdot4.50\approx 288.
\]
At $\gamma=0.05$, the confidence contribution in
Theorem~\ref{thm:calm-main} is therefore
\[
C\frac{\log(1/\gamma)}{n^r}
\leq
288\frac{\log(20)}{500}
\approx
1.73.
\]
This calibration is intentionally conservative: it includes the
worst-case $1/\rho^2$ overlap factor, a coordinatewise union bound, and a
non-sharp universal constant.

\subsection{Reporting the Lipschitz constants}

The constants in Assumption~\ref{as:alignment} are prediction-head
constants: $L_\mu$ controls the OS outcome head and $L_\delta$ controls
the calibration head.  For the diagnostics in this paper we report
spectral-product upper bounds for those heads only.  The neural encoders
also have formal spectral-product bounds, but those values are vacuous in
the residual-MLP architecture because conservative LayerNorm handling
multiplies by $\max|\gamma|/\sqrt{\varepsilon}$; we therefore do not
report them as $L_\mu$ or $L_\delta$.

\begin{table}[t]
\centering
\caption{Spectral-product upper bounds for the outcome-head
($L_\mu$) and calibration-head ($L_\delta$) maps in \CALM-NN\@.  Treated
corresponds to $a=+1$ and control to $a=-1$.  On GPS the observational
cohort has a single (control) arm, so the treated-arm outcome head is not
learned from OS data and is marked n/a.}
\label{tab:cr-lipschitz}
\begin{tabular}{@{}lcccc@{}}
\toprule
Setting & $L_{\mu,+1}$ & $L_{\mu,-1}$ & $L_{\delta,+1}$ & $L_{\delta,-1}$ \\
\midrule
Synthetic baseline linear & 0.936 & 0.866 & 0.650 & 0.591 \\
Synthetic nonlinear private-$\bV$ & 0.949 & 0.847 & 0.780 & 0.690 \\
GPS real data, EHR-restricted & n/a & 0.548 & 0.914 & 0.541 \\
\bottomrule
\end{tabular}
\end{table}

These are upper bounds, not exact Lipschitz constants.  The neural
stages use Adam with weight decay $10^{-4}$ but no spectral
normalization, so the values should be interpreted as
regularization-controlled diagnostics rather than certified global
constants.  Encoder smoothness is controlled only indirectly by the same
weight decay and could be tightened with norm-constrained architectures.
Exact Lipschitz computation for ReLU networks is NP-hard in general;
post-hoc certified bounds would require methods such as those of
\citet{anil2019sorting} and \citet{fazlyab2019efficient}.

\end{document}